\documentclass[hidelinks,onefignum,onetabnum,tikz]{siamart250211}

\usepackage{lipsum}
\usepackage{amsfonts}
\usepackage{graphicx}
\usepackage{epstopdf}
\usepackage{algorithmic}
\usepackage{xr}
\usepackage{hyperref}
\usepackage{tcolorbox}
\ifpdf
  \DeclareGraphicsExtensions{.eps,.pdf,.png,.jpg}
\else
  \DeclareGraphicsExtensions{.eps}
\fi

\usepackage{amsmath,bm}
\usepackage{amssymb}
\usepackage{mathtools}
\usepackage{xfrac}
\usepackage{tikz}
\usetikzlibrary{calc,angles,quotes,arrows.meta}

\newcommand{\Rea}{\mathbb{R}}
\newcommand{\paren}[1]{\left( #1 \right)}
\newcommand{\set}[1]{\left\{ #1 \right\}}

\newcommand{\norm}[1]{\left\lVert #1 \right\rVert}
\newcommand{\expectE}{\mathbb{E}\,}
\newcommand{\expect}[2]{\mathbb{E}_{#1}\left[#2\right]}
\newcommand{\range}{\operatorname{range}}

\DeclareMathOperator{\argm}{argmin}
\newcommand{\argmin}[1]{\underset{#1}{\argm\;}}

\newcommand{\rev}[1]{#1}
\newcommand{\revv}[1]{#1}
\newcommand{\revh}[1]{#1}

\newcommand{\xrh}{x^{(\rho)}}
\newcommand{\rrh}{r^{(\rho)}}
\newsiamremark{remark}{Remark}
\newsiamremark{hypothesis}{Hypothesis}
\crefname{hypothesis}{Hypothesis}{Hypotheses}
\newsiamthm{claim}{Claim}
\newsiamremark{fact}{Fact}
\crefname{fact}{Fact}{Facts}
\newsiamthm{question}{Question}
\newsiamthm{example}{Example}

\headers{Randomized block Kaczmarz without Preprocessing}{G. Goldshlager, J. Hu, and L. Lin}

\title{Worth Their Weight: Randomized and Regularized Block Kaczmarz Algorithms without Preprocessing \thanks{G.G. acknowledges support from the U.S. Department of
Energy, Office of Science, Office of Advanced Scientific Computing Research, Department of
Energy Computational Science Graduate Fellowship under Award Number DE-SC0023112.
This effort was supported by the SciAI Center, and funded by the Office of Naval Research (ONR), under Grant Number N00014-23-1-2729 (J.H., L.L.).  L.L. is a Simons Investigator in Mathematics. This research used the Savio computational cluster resource provided by the Berkeley Research Computing program at the University of California, Berkeley.}}

\author{Gil Goldshlager \thanks{Department of Mathematics, UC Berkeley, Berkeley, CA 94720 USA
  (\email{ggoldsh@berkeley.edu, jianghu@berkeley.edu, linlin@berkeley.edu}).}
\and Jiang Hu \footnotemark[2] 
\and Lin Lin \footnotemark[2] 
\thanks{Applied Mathematics and Computational Research Division, Lawrence Berkeley National Laboratory, Berkeley, CA 94720 USA (\email{linlin@lbl.gov}).}
}

\usepackage{amsopn}

\ifpdf
\hypersetup{
  pdftitle={Worth Their Weight: Randomized and Regularized Block Kaczmarz Algorithms without Preprocessing},
  pdfauthor={G. Goldshlager, J. Hu, and L. Lin}
}
\fi

\begin{document}

\maketitle

\begin{abstract}
Due to the ever growing amounts of data leveraged for machine learning and scientific computing, it is increasingly important to develop algorithms that sample only a small portion of the data at a time.
In the case of linear least-squares, the randomized block Kaczmarz method (RBK) is an appealing example of such an algorithm, but its convergence is only understood under sampling distributions that require potentially prohibitively expensive preprocessing steps.
To address this limitation, we analyze RBK when the data is sampled uniformly, showing that its iterates converge in a Monte Carlo sense to a \textit{weighted} least-squares solution.
Unfortunately, for general problems \revv{the bias of the weighted least-squares solution} and the variance of the iterates can become arbitrarily large.
\revv{We show that these quantities can be rigorously controlled} by incorporating regularization into the RBK iterations, yielding the regularized algorithm ReBlocK.
Numerical experiments including examples arising from natural gradient optimization demonstrate that ReBlocK can outperform both RBK and minibatch stochastic gradient descent for inconsistent problems with rapidly decaying singular values.
\end{abstract}



\section{Introduction}
Consider the linear least-squares problem 
\begin{equation} \label{eq:ls}
\begin{split}
& \min_{x \in \Rea^n} \norm{Ax-b}^2, \\
A  &  \in \Rea^{m \times n} ,\; b \in \Rea^m.
\end{split}
\end{equation}
The minimal-norm ordinary least-squares solution is 
\begin{equation}
x^* = A^+b,
\end{equation}
where $A^+$ is the Moore-Penrose pseudoinverse of $A$. 

In this work, we are interested in algorithms that solve \eqref{eq:ls} by sampling just a small number of rows at a time. Such algorithms can be useful for extremely large problems for which direct methods and Krylov subspace methods are prohibitively expensive \cite{censor1981row}. Of particular interest are problems in which the dimension $n$ is so large that it is only possible to access $k \ll n$ rows at a time, but not so large that it is necessary to take $k=1$. As a canonical example, we might have $m=10^9$, $n=10^6$, and $k=10^3$.

As further motivation for algorithms of this type, there are some applications in which sampling rows is the only efficient way to access the data. For example, the rows may be computed on the fly, especially when solving the ``semi-infinite'' version of \eqref{eq:ls} in which the rows are indexed by continuous variables rather than discrete integers \cite{shustin2022semi}. An example that captures both of these features is the problem of calculating natural gradient directions for continuous function learning problems; see \Cref{sec:ngd}.

One well-known approach for solving \eqref{eq:ls} using only a few rows at a time is minibatch stochastic gradient descent (mSGD). The mSGD algorithm works by sampling $k$ rows uniformly at random and using them to calculate an unbiased estimator of the gradient of the least-squares loss function. This is equivalent to averaging the $k$ independent gradient estimators furnished by each sampled row. For a thorough discussion of mSGD, see \cite{jain2018parallelizing}.

The technique of averaging used in mSGD, while simple and efficient, is a relatively crude way of processing a block of $k$ rows. A more sophisticated approach is provided by the randomized block Kaczmarz method, which goes beyond averaging by making use of the pseudoinverse of the sampled block \cite{needell2014paved}. As we shall see, the use of the pseudoinverse opens up the possibility of faster convergence, but it also creates a number of complications regarding the implementation and analysis of the algorithm.

\subsection{Notation}
We denote by $r = b - Ax^*$ the residual vector of \eqref{eq:ls}, by $a_i^\top \in \Rea^n$ the row of $A$ with index $i$, and by $b_i \in \Rea$ the corresponding entry of $b$. Additionally, for an index set $S \subseteq \set{1, \ldots, m}$ with $|S| = k$, let $A_S \in \Rea^{k \times n}$ represent the block of rows of $A$ whose indices are in $S$, and let $b_S \in \Rea^k$ represent the corresponding entries of $b$. The same subscript notations will also be applied as needed to any matrices and vectors other than $A$ and $b$. Additionally, denote by $\mathbf{U}(m,k)$ the uniform distribution over all size-$k$ subsets of $\set{1, \ldots, m}$.
 
For symmetric matrices $X,Y$, denote by $X \succ Y$ that the difference $X-Y$ is positive definite, and define $\succeq, \prec, \preceq$ correspondingly. For any vector $x$ and any positive definite matrix $Y$ of the same size, let $\norm{x}_Y = \sqrt{x^\top Y x}$.  For any matrix $X$, denote by $\|X\|_F$ its Frobenius norm and by $\sigma_{\rm min}^+(X)$ its minimum nonzero singular value.

\subsection{Randomized Kaczmarz\label{sub:rk_intro}}
The modern version of the randomized Kaczmarz method (RK) was proposed by Thomas Strohmer and Roman Vershynin in 2009 \cite{strohmer2009randomized}. In RK, an initial guess $x_0$ is updated iteratively using a two-step procedure:
\newline
\begin{enumerate}
    \item \textit{Sample} a row index $i_t \in \set{1, \ldots, m}$ with probability proportional to $\norm{a_{i_t}}^2$.
    \item \textit{Update} 
\begin{equation}
\label{eq:rk}
x_{t+1} = x_t + a_{i_t} \frac{b_{i_t} - a_{i_t}^\top x_t}{\norm{a_{i_t}}^2}.
\end{equation}
\end{enumerate}

To reduce the sampling cost, it is also possible to run RK with uniform sampling, which is equivalent to running RK on a diagonally reweighted problem \cite{needell2014stochastic}.

The iteration \eqref{eq:rk} has the interpretation of projecting $x_t$ onto the hyperplane of solutions to the individual equation $a_{i_t}^\top x = b_{i_t}$. For consistent systems, $Ax^*=b$, the RK iterates $x_t$ converge linearly to $x^*$ with a rate that depends on the conditioning of $A$. For inconsistent systems, $Ax^* \neq b$, RK converges only to within a finite horizon of the ordinary least-squares solution $x^*$ \cite{needell2010randomized}. In particular, the expected squared error $\expectE \norm{x_t - x^*}^2$ converges to a finite, nonzero value that depends on the conditioning of $A$ and the norm of the residual vector $r = b - Ax^*$. See Theorem 7 of \cite{zouzias2013randomized} for the strongest known convergence bound of this type. 

\subsection{Tail Averaging}
Tail averaging is a common technique for boosting the accuracy of stochastic algorithms \cite{rakhlin2011making, jain2018parallelizing, epperly2024randomizedkaczmarztailaveraging}. Given a series of stochastic iterates $x_0, \ldots, x_T$ and a burn-in time $T_b$, the tail-averaged estimator is given by
\begin{equation}
\label{eq:tail}
\overline{x}_T = \frac{1}{T - T_b} \sum_{t=T_b+1}^{T} x_t.
\end{equation}
The recent work \cite{epperly2024randomizedkaczmarztailaveraging} shows that applying tail averaging to the RK iterates yields exact convergence (with no finite horizon) to the ordinary least-squares solution $x^*$, even for inconsistent systems. Building on these results, we will make use of tail averaging to obtain exact convergence to a \textit{weighted} least-squares solution in the block case. 

\subsection{Randomized Block Kaczmarz\label{sub:rbk_intro}}
Randomized block Kaczmarz (RBK) is an extension of RK which uses blocks of rows to accelerate the convergence and make better use of parallel and distributed computing resources \cite{elfving1980block,needell2014paved}. Like RK, each RBK iteration proceeds in two steps:
\begin{enumerate}
    \item \textit{Sample} a subset $S_t \subset \set{1, \ldots, m}$ of the row indices from some chosen sampling distribution $\rho$.
    \item \textit{Update}
\begin{equation}
x_{t+1} = x_t + A_{S_t}^+ (b_{S_t} - A_{S_t} x).
\label{eq:rbk}
\end{equation}
\end{enumerate}
Here $A_{S_t}^+$ is the Moore-Penrose pseudoinverse of $A_{S_t}$. The iteration \eqref{eq:rbk} has the interpretation of projecting $x_t$ onto the hyperplane of solutions to the block of equations $A_{S_t} x = b_{S_t}$. The RBK method can also be viewed as a ``sketch-and-project'' algorithm; see \cite{gower2015randomized}.

There have been many proposals for how to choose the blocks in the RBK method. One idea is to use a preprocessing step to partition the matrix $A$ into well-conditioned blocks, then sample this fixed set of blocks uniformly \cite{needell2014paved}. It has also been suggested to preprocess the matrix with an incoherence transform, which can make it easier to generate a well-conditioned partition \cite{needell2014paved} or, relatedly, enable RBK with uniform sampling to converge rapidly for the transformed problem \cite{derezinski2024solving}. The work of \cite{derezinski2024solving} also provides an analysis of the RBK algorithm when sampling from a determinantal point process, and other proposals include greedy block Kaczmarz algorithms such as \cite{liu2021greedy} which require evaluating the complete residual vector at each iteration. 

Unfortunately most of these proposals apply only to consistent linear systems, and all of them require at least a preprocessing step in which the entire data matrix must be accessed. Such preprocessing can be prohibitively expensive for very large-scale problems, for which 1) it can be necessary to furnish an approximate solution without processing the entire data set even once (for instance, in the semi-infinite case), and 2) it can be impossible to manipulate more than a tiny subset of the data at a time due to storage constraints. This leads us to the central question of our work: 
\begin{tcolorbox}[colback=green!10!white,colframe=green!60!black]
  \begin{question} \label{ques:main}
      Can RBK, or some variant thereof, be applied to solve inconsistent linear systems without preprocessing the input matrix?
  \end{question} 
\end{tcolorbox}

\revv{To avoid preprocessing, we consider the case that the sampling strategy is chosen \emph{a priori} in that it is not necessarily adapted to the data matrix $A$ or the right-hand side $b$. To simplify the analysis, we focus concretely on the case when the sampling is uniform. This is a natural choice since \eqref{eq:ls} can be viewed as a uniform mixture of $m$ distinct rows. We emphasize that this choice is made for simplicity and that our results are not inherently restricted to the case of uniform sampling. For example,} our results can be readily generalized to both weighted and semi-infinite problems of the form
\begin{equation}
\min_{x \in \Rea^n} \expect{i \sim \mu}{(a_i^\top x - b_i)^2},
\end{equation}
for which the corresponding algorithms would independently sample $k$ indices $i_1, \ldots, i_k \sim \mu$. This is especially relevant for scientific applications, in which data is often continuous and may be sampled using a physics-based probability distribution. For example, in the case of neural network wavefunctions \cite{hermann2023ab}, the sampling distribution is known as the Born probability density.

\subsection{Contributions}
\rev{
In this work, we make substantial progress towards answering our central question \ref{ques:main}. We first provide a new analysis of the RBK algorithm under uniform sampling which shows both the strengths and weaknesses of the algorithm. Inspired by this analysis, we propose and analyze a regularized algorithm, ReBlocK, which behaves more robustly for general types of data.  Concretely, we summarize our contributions as follows:
\begin{enumerate}    
\item \textbf{We demonstrate that the RBK algorithm with uniform sampling (RBK-U) converges in a Monte Carlo sense to a weighted least-squares solution for both consistent and inconsistent linear systems}. In particular, \cref{thm:rbk} shows that convergence is obtained by both expectation values of individual iterates and tail averages of the sequence of iterates. The weight matrix depends on the block size $k$ and the matrix $A$, but not on the vector $b$. Our results provide a new perspective on  RBK in the inconsistent case, going beyond previous analyses that only characterized proximity to the ordinary least-squares solution. 
\item  \textbf{We provide a new perspective on the pitfalls of RBK with uniform sampling.} Concretely, our analysis reveals that when the problem contains many blocks $A_S$ that are nearly singular, \revv{the bias of the weighted least-squares solution} and the variance of the iterates can become arbitrarily large. \revv{See \cref{thm:rbk,thm:no-go} for theoretical barriers and \cref{fig:rbku_fail_examples} for numerical examples that manifest these issues.}
\item  \textbf{We show that RBK-U performs robustly and efficiently for Gaussian data.} Concretely, \cref{thm:rbk_gaussian} shows that convergence is obtained to the ordinary least-squares solution when the data arises from a multivariate Gaussian distribution. Furthermore, in this case both the variance of the iterates and the convergence parameter $\alpha$ can be explicitly bounded. When the singular values of the covariance matrix decay rapidly, the convergence rate can be much faster than mSGD; see \cref{coro:gauss} and \cref{fig:gaussian_results}.
\item  \textbf{We propose and analyze a regularized algorithm to handle more general types of data.} To provide a more general solution, we propose to regularize the RBK iterations as follows:
\begin{equation}
x_{t+1} = x_t + A_{S_t}^\top (A_{S_t} A_{S_t}^\top + \lambda k I)^{-1} (b_{S_t} - A_{S_t} x_t)
\label{eq:reblock_it}.
\end{equation}
We refer to this algorithm as the regularized block Kaczmarz method, or ReBlocK. Similar to RBK-U, we show that ReBlocK with uniform sampling (ReBlocK-U) converges in a Monte Carlo sense to a weighted least-squares solution; see \cref{thm:reblock}. Unlike for RBK-U, \revv{the bias of the weighted least-squares solution} and the variance of the iterates can be controlled in terms of just $\lambda$ and some coarse properties of the data $A,b$. This makes ReBlocK-U much more reliable than RBK-U in practice; see \cref{fig:reblock_success,fig:nn}. As an added benefit,  ReBlocK iterations can be significantly more efficient than RBK iterations; see \Cref{sec:reblock_imp,fig:it_speed}. 
\item \textbf{We demonstrate promising results for using ReBlocK to calculate natural gradient directions.}
Our initial motivation for this work came from the problem of calculating natural gradient directions for deep neural networks. In \Cref{sec:ngd}, we explain how this setting naturally lends itself to the kinds of linear least-squares solvers that we study in this paper. Encouragingly, \cref{fig:nn} shows that ReBlocK-U outperforms mSGD and RBK-U when calculating natural gradient directions for a simple neural function regression task. 
\end{enumerate}
Altogether, our results demonstrate that regularization can improve the robustness of the randomized block Kaczmarz algorithm and unlock new applications to challenging problems for which sophisticated preprocessing schemes are out of reach.  Furthermore, while our focus is on the case of uniform sampling, the same proof techniques can be applied to other sampling distributions. For example, in \cref{app:dpp} we show that sampling from an appropriate determinantal point process can in theory enable tail averages of ReBlocK iterates to converge rapidly to the ordinary least-squares solution for arbitrary inconsistent problems.}

\subsection{Related Works}
Our regularized algorithm, ReBlocK, is closely related to the iterated Tikhonov-Kaczmarz method \cite{de2011modified}. ReBlocK can also be viewed as a specific application of stochastic proximal point algorithms (sPPA) \cite{bertsekas2011incremental,asi2019stochastic,davis2019stochastic} for solving stochastic optimization problems with objective functions of the least-squares type. To ensure the exact convergence of sPPA, diminishing step sizes are required; see for example \cite{puatracscu2021new}. In contrast, our work investigates the convergence of sPPA with a large constant step size for the special case of a least-squares loss. Such an approach allows for aggressive updates throughout the algorithm, potentially improving its practical efficiency. 

\rev{ The ReBlocK algorithm is additionally a special case of the SlimLS algorithm proposed by 
\cite{chung2020sampled} for solving massive linear inverse problems. In fact, their convergence result for the expectation of the ReBlocK iterates, found in result (a)  of their Theorem 3.1, is equivalent to  equation \eqref{eq:reblock_exp} of our \cref{thm:reblock}. However, by identifying the limiting expectation value as a weighted least-squares solution, we provide greater conceptual clarity regarding its meaning. Additionally, our bounds regarding the variance of the ReBlocK iterates, the convergence of tail averages, and the suboptimality of the weighted least-squares solution are all much stronger than the corresponding bounds of \cite{chung2020sampled}. These stronger bounds serve to explain our numerical findings that ReBlocK can be robust and efficient even when $\lambda$ is very small and the problem contains many nearly singular blocks $A_S$, which is the most relevant case and is not explained by the previous analysis. }

Our work is also related to the nearly concurrent paper \cite{derezinski2025}, which introduces Tikhonov regularization into the RBK iterations just like ReBlocK. \cite{derezinski2025} focuses on consistent systems and in this context, the regularization gives rise to optimal convergence rates in the presence of Nesterov acceleration. On the other hand, our work focuses on solving inconsistent systems, in which case the regularization is needed to ensure the stability of the algorithm. 
Exploring the combination of regularization and Nesterov acceleration in the inconsistent case is a promising direction for future work.

\section{Overview of the Analysis\label{sec:rbk}}

Our results are stated in terms of a unified framework that includes RBK, ReBlocK, and even mSGD as special cases. Consider the following iteration: \newline
\begin{enumerate}
    \item \textit{Sample} $S_t \sim \rho$.
    \item \textit{Update } 
\begin{equation} x_{t+1} = x_t + A_{S_t}^\top M(A_{S_t}) (b_{S_t} - A_{S_t} x_t) \label{eq:it_general}. \end{equation}
\end{enumerate}
Here $M(\cdot): \Rea^{k \times n} \rightarrow \Rea^{k \times k}$ takes in the sampled block $A_{S_t}$ and returns a positive semidefinite ``mass'' matrix. RBK is recovered by setting $M(A_{S_t}) = (A_{S_t} A_{S_t}^\top)^+$ and ReBlocK by setting $M(A_{S_t}) = (A_{S_t} A_{S_t}^\top + \lambda k I)^{-1}$. See \cref{alg:gen} for the full procedure with and without tail averaging. 

Once the function $M(A_S)$ is chosen, let
\begin{equation}
W(S)= I_{S}^\top M(A_{S}) I_{S},\;
P(S) = A_{S}^\top  M(A_{S}) A_{S},
\end{equation}
where $I_S$ represents the $k$ rows of the $n \times n$ identity matrix whose indices are in $S$ (recall $|S|=k$). These quantities are natural because they enable the general iteration \eqref{eq:it_general} to be rewritten as 
\begin{equation}
x_{t+1} = (I - P(S_t)) x_t + A^\top W(S_t) b.
\end{equation}
Note that $P(S_t)$ is a projection matrix in the case of RBK.

Next, let
\begin{equation}
\overline{W} = \expect{S \sim \rho}{W(S)},\; \overline{P} =  \expect{S \sim \rho}{P(S)}.
\end{equation}
Additionally, define the weighted solution $\xrh$ and the weighted residual $\rrh$ via
\begin{equation} 
\label{eq:xrh}
\xrh = \argmin{x \in \Rea^n} \norm{Ax-b}_{\overline{W}}^2, \;
\rrh = b - A\xrh.
\end{equation}
When the solution to the weighted problem is not unique, let $\xrh$ refer to the minimal-norm solution.

Using these definitions, we can further rewrite the iteration \eqref{eq:it_general} as
\begin{equation}
\label{eq:it_nice}
x_{t+1} - \xrh = (I - P(S_t)) (x_t - \xrh) + A^\top W(S_t) \rrh.
\end{equation}

Since the normal equations for \eqref{eq:xrh} can be written as $A^\top \overline{W} \rrh = 0$, we observe that the final term in \eqref{eq:it_nice} vanishes in expectation. Indeed, identifying the appropriate weighted solution $\xrh$ to enable the generalized iteration \eqref{eq:it_general} to be written in the form of \eqref{eq:it_nice}, namely as a linear contraction of the error plus a zero-mean additive term, is the main technical innovation underlying our results. From here, the analysis of \cite{epperly2024randomizedkaczmarztailaveraging} can be readily generalized to show that convergence to $\xrh$ is obtained.

\begin{algorithm}[tb]
   \caption{Generalized iterative least-squares solver with optional tail averaging}
   \label{alg:gen}
\begin{algorithmic}
   \STATE {\bfseries Input:} Data $A,b$, block size $k$, initial guess $x_0$
   \STATE {\bfseries Input:} Mass matrix $M(A_S)$, sampling distribution $\rho$
   \STATE {\bfseries Input:} Total iterations $T$, optional burn-in time $T_b$
   \FOR{$t=0$ {\bfseries to} $T-1$}
   \STATE Sample $S_t \sim \rho$
   \STATE $x_{t+1} = x_t + A_{S_t}^\top M(A_{S_t}) (b_{S_t} - A_{S_t} x)$
   \ENDFOR
    \vspace{0.1cm}
   \IF{$T_b$ is not provided} 
           \STATE {\bfseries Return} $x_T$
   \ENDIF
  \vspace{0.1cm}
   \STATE $\overline{x}_{T} = \frac{1}{T - T_b} \sum_{t=T_b+1}^{T} x_t$
   \vspace{0.1cm}
   \STATE {\bfseries Return} $\overline{x}_T$
\end{algorithmic}
\end{algorithm}

\section{RBK without Preprocessing \label{sec:rbku}}
\revv{
We begin by presenting our general convergence bounds for RBK-U. We emphasize that these bounds should not be viewed as a proof that algorithm is effective, but rather as a helpful framework to understand the properties and the potential problems of the algorithm. Indeed we will later demonstrate severe failure modes of RBK-U that are inspired by the analysis (\Cref{thm:no-go}, \Cref{fig:rbku_fail_examples}) and we will ultimately show that regularization is needed in order to eliminate these failure modes and obtain satisfying bounds on the bias and the variance (\Cref{thm:reblock}).}

\begin{theorem}
\label{thm:rbk}
Consider the RBK-U algorithm, namely \cref{alg:gen} with $M(A_S) = (A_S A_S^\top)^+$ and  $\rho = \mathbf{U}(m,k)$. Let $\alpha = \sigma^+_{\rm min}(\overline{P})$ and assume that $x_0 \in \range(A^\top)$. Then the expectation of the RBK-U iterates $x_T$ converges to $\xrh$ as
\begin{equation}
\label{eq:rbk_expect_converge}
\norm{\expect{}{x_T} - \xrh} \leq (1 - \alpha)^T \norm{x_0 - \xrh}.
\end{equation}
Furthermore, the tail averages $\overline{x}_T$ converge to $\xrh$ as
\begin{equation}\label{eq:rbk_ta_converge}
\expectE \norm{\overline{x}_T - \xrh}^2 \leq  \paren{1 - \alpha}^{T_b+1} \norm{x_0 - \xrh}^2 + \frac{2}{\alpha^2 (T-T_b)} \; \revv{V}
\end{equation}
\revv{with $V = \expectE_{S \sim \rho} \|A_S^+ \rrh_S\|^2$. Finally, when the rows of $A$ are in general position the condition number $\kappa(\overline{W})$, the weighted residual $r^{(\rho)}$, the bias $x^{(\rho)} - x^*$, and the variance $V$ are bounded as
\begin{equation}
\begin{aligned}
\kappa(\overline{W}) &\leq k \cdot \max_i \norm{a_i}^2 \cdot \max_S \|A_S^+\|^2, \\
\|r^{(\rho)}\| &\leq \sqrt{\kappa(\overline{W})} \cdot \norm{r}, \\ 
 \|x^{(\rho)} - x^*\| &\leq \sqrt{\kappa(\overline{W}) -1} \cdot 
\norm{A^+} \cdot \norm{r}, \\
V & \leq \max_S \|A_S^+\|^2 \cdot \kappa(\overline{W}) \cdot \frac{k\norm{r}^2}{m}.
\end{aligned}
\end{equation}
}
\end{theorem}
To our knowledge, this is the first result for inconsistent linear systems that characterizes the exact solution to which the randomized block Kaczmarz iterates converge, albeit in a Monte Carlo sense. The $O(1/T)$ convergence rate for the tail-averaged bound is optimal for row-access methods, and a reasonable default for the burn-in time is $T_b = T/2$; see \cite{epperly2024randomizedkaczmarztailaveraging} for a more thorough discussion of these points. The proof of \cref{thm:rbk}, which takes advantage of an orthogonal decomposition of the error term $x_{t+1} - \xrh$, can be found in \cref{app:rbk}.

Unfortunately, \cref{thm:rbk} does not imply robust convergence for general problems. Indeed, for problems containing nearly singular blocks $A_S$, \revv{the bounds on the condition number, the weighted residual, the bias, and the variance can be arbitrarily large. In fact, these issues are not just defects of the analysis but real problems with RBK-U. To substantiate this, we present a single-parameter family of $3 \times 2$ examples that completely breaks RBK-U:}

\revv{
\begin{example} [No-go for RBK-U]\label{prop:rbk_fail}
Let $\epsilon > 0$ and
\begin{equation}
A = \begin{bmatrix}
0 & 1 \\
1 & \epsilon^2 \\
1  & -\epsilon^2 
\end{bmatrix}, \quad 
b = \begin{bmatrix}
0 \\ 1 + \epsilon \\ 1 - \epsilon
\end{bmatrix}. \label{eq:isosceles}
\end{equation} 
Then $\lim_{\epsilon \rightarrow 0^+} \kappa(A) = \sqrt{2}$ and $\lim_{\epsilon \rightarrow 0^+} \norm{r} = 0$, indicating that the problem is well conditioned and approaches consistency as $\epsilon \rightarrow 0$.
On the other hand, as $\epsilon \rightarrow 0$ the bottom $2\times2$ block approaches singularity and so for $k=2$ the upper bounds of $\kappa(\overline{W})$, $\norm{r^{(\rho)}}$, $\norm{x^{(\rho)} - x^*}$, and $V$ from \Cref{thm:rbk} all diverge in this limit. In fact, this divergence is necessary as it can be verified that for RBK-U with $k=2$, it holds \newline
\begin{equation}
\lim_{\epsilon \rightarrow 0^+} \kappa(\overline{W}) = 
\lim_{\epsilon \rightarrow 0^+} \|r^{(\rho)}\| =
\lim_{\epsilon \rightarrow 0^+} \|x^{(\rho)} - x^* \| = \lim_{\epsilon \rightarrow 0^+} V = \infty.
\end{equation}
\label{thm:no-go}
\end{example}
The fact that the bias \textit{and} variance of the RBK-U iterates diverges while the underlying problem is approaching consistency represents a clear and catastrophic failure of the algorithm. We can understand this failure by noting that geometrically, this problem represents three constraints in two dimensions which together form an isosceles triangle with vertices at $(1 -\epsilon, 0)$, $(1 + \epsilon,0)$, and $(1, 1/\epsilon)$. When running RBK-U with $k=2$, the algorithm jumps between the three vertices of the triangle with equal probabilities, and as a result $x^{(\rho)}$ coincides with the centroid of the triangle. The problem is that as $\epsilon \rightarrow 0$ the top vertex walks off to infinity and drags the centroid with it, while the least-squares solution $x^*$ converges to $(1,0)$. This insight can be used to directly show the divergence of the bias, the residual, and the variance.  The residual bound $\|r^{(\rho)}\| \leq \sqrt{\kappa(\overline{W})} \cdot \norm{r}$ then implies the divergence of $\kappa(\overline{W})$ since $\|r^{(\rho)}\|$ diverges while $\norm{r}$ goes to zero. The situation is depicted visually in \cref{fig:rbku_fail}.
}

\begin{figure}[htbp]
  \centering
  \begin{tikzpicture}[scale=0.85, line cap=round, line join=round]
    \definecolor{niceblue}{RGB}{30,90,200}
    \definecolor{nicegreen}{RGB}{20,170,60}
    \definecolor{errorred}{RGB}{255,0,0}
    \tikzset{
      constraint/.style={line width=0.9pt, draw=niceblue},
      motionarrow/.style={-{Latex[length=2.5mm]}, line width=0.9pt, draw=nicegreen},
      errorarrow/.style={-{Latex[length=2.5mm]}, line width=0.9pt, draw=errorred},
      bluelabel/.style={text=niceblue}
    }

    \path[use as bounding box] (-3,-1) rectangle (3,6);
    \begin{scope}
      \clip (-3,-1) rectangle (3,6);

      \draw[constraint] (-10,0) -- (10,0);     
      \draw[constraint] (-10,55) -- (10,-45);  
      \draw[constraint] (-10,-45) -- (10,55);  
    \end{scope}

    \path (-10,0) -- (10,0)
      node[pos=0.50, below=6pt, bluelabel] {$\bm a_{1}^{\top}\bm x = b_{1}$};
    \path (0,5) -- (1,0)
      node[pos=0.60, sloped, above=6pt, bluelabel] {$\bm a_{2}^{\top}\bm x = b_{2}$};
    \path (0,5) -- (-1,0)
      node[pos=0.60, sloped, above=6pt, bluelabel] {$\bm a_{3}^{\top}\bm x = b_{3}$};

    \coordinate (Apex) at (0,5);
    \coordinate (Left) at (-1,0);
    \coordinate (Right) at (1,0);
    \coordinate (xstar) at (0,{5/14});
    \coordinate (xr)    at (0,{5/3});

    \foreach \pt in {(Apex),(Left),(Right)}{\fill[niceblue] \pt circle (1.6pt);}

    \fill[nicegreen] (xstar) circle (1.4pt) node[above=1pt] {$x^{\ast}$};
    \fill[errorred] (xr)    circle (1.4pt) node[below=1pt] {$x^{(\rho)}$};

    \draw[errorarrow] (xr)    -- ++(0,0.8);
    \draw[motionarrow] (xstar) -- ++(0,-0.5);
    \draw[motionarrow] (Apex)  -- ++(0,1.0);
    \draw[motionarrow] (Left)+(0.01,0.05) -- ++(0.7, 0.05);
    \draw[motionarrow] (Right)+(0.01,0.05) -- ++(-0.7, 0.01);
  \end{tikzpicture}
  \caption{\revv{Visual depiction of the $3 \times 2$ linear system \cref{eq:isosceles} that causes RBK-U to fail catastrophically when $k=2$. As $\epsilon \rightarrow 0$ the top vertex walks off to infinity and the recovered solution $x^{(\rho)}$ goes with it, while the true solution $x^*$ approaches the $x$-axis.}} 
  \label{fig:rbku_fail}
\end{figure}

\revv{It is worth noting that such problems are not limited to any particular batch size such as $k=2$. More generally we can view the skinny isosceles triangle as a stretched $2$-simplex and indeed, the same arguments can be readily generalized to show that for arbitrary $k$, a similarly stretched $k$-simplex suffices to break RBK-U with a block size of $k$. The problem is also not directly related to the use of uniform sampling, and similar problems could arise for any \emph{a priori} sampling distribution.}

\subsection{Numerical demonstration\label{sec:rbk_fail_num}}
\revv{We now demonstrate that the issues highlighted by \Cref{thm:no-go} can occur for more realistic problems. To do so, we set $m=10^5$, $n=10^2$ and construct two inconsistent problems whose columns are discretized representations of continuous functions,
leading the matrices A to contain many nearly singular blocks of rows. We apply both minibatch SGD (mSGD) and RBK-U to the resulting problems, applying tail averaging to each algorithm to observe convergence beyond the variance horizon.
The results in \cref{fig:rbku_fail_examples} confirm that RBK-U performs poorly for these problems, just as our theory suggests. In fact  RBK-U is not even able to reliably attain a relative error of $1.0$ in either case, meaning the algorithm performs worse than simply guessing $x^* = 0$. For these and other experiments in \Cref{sec:rbku,sec:gauss_data,sec:reblock} we choose the burn-in time $T_b$ to demonstrate the different phases of convergence as clearly as possible.}

\revv{We note that for such realistic examples, it is difficult to pin down exactly how large the bias $x^{(\rho)} - x^*$ is. The reason is that the variance $V$ is so large that running the algorithm to convergence would be prohibitively expensive. Nonetheless it is evident that RBK-U is not effective for these problems, even after incorporating tail averaging.} The code for all of our experiments can be found at \url{https://github.com/ggoldsh/block-kaczmarz-without-preprocessing}, and further details on these particular experiments can found in \Cref{app:synth}.

\begin{figure}[htbp]
\centering
{\includegraphics[width=0.4\textwidth]{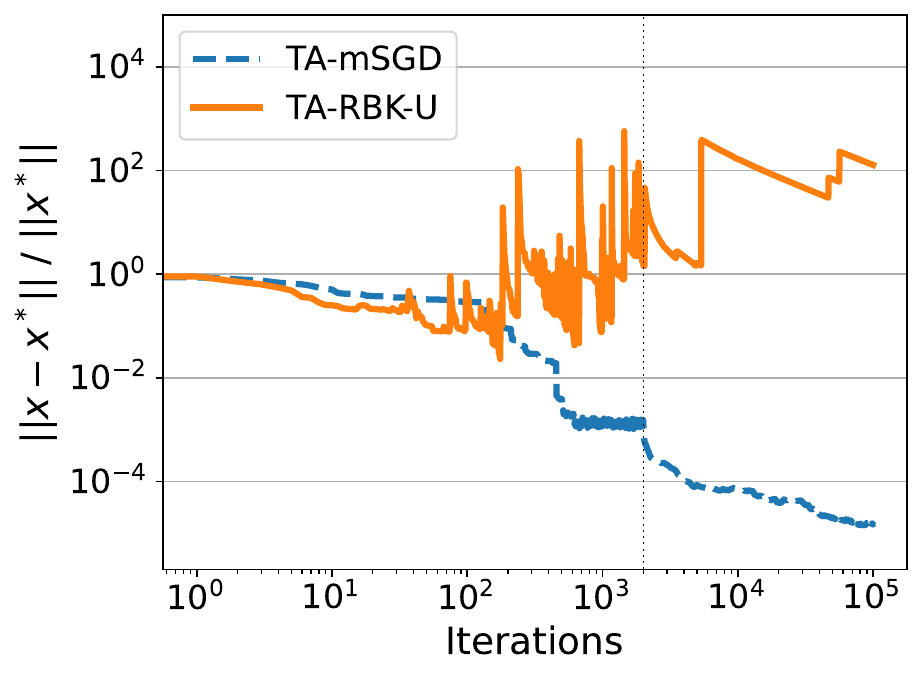}}
\quad
{\includegraphics[width=0.4\textwidth]{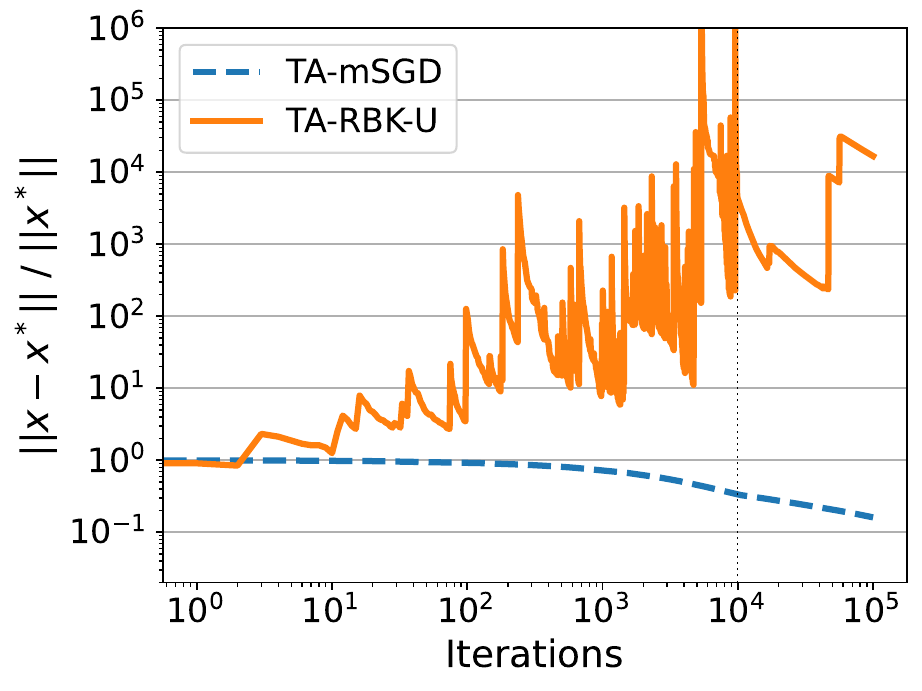}}
\caption{\revv{Failure of RBK-U} for two inconsistent problems with many nearly singular blocks, with mild singular value decay (left) and rapid singular value decay (right). The vertical dotted line indicates the burn-in time, before which results are shown for individual iterates.}
\label{fig:rbku_fail_examples}
\end{figure}

\subsection{Implementation Details}

To stably implement the RBK iteration \eqref{eq:rbk}, we employ a \revv{QR-based} least-squares solver to calculate $A_{S_t}^+ (b_{S_t} - A_{S_t} x).$  The most expensive part of this procedure is the \revv{QR decomposition}, which has an asymptotic cost of $O(nk^2)$. This cost is greater than the cost of an mSGD iteration, which is $O(nk)$.

\section{Linear Least-squares with Gaussian Data\label{sec:gauss_data}}

\revv{Given the problems with RBK-U for general data, it is natural to wonder if there are more restricted categories of benign problems for which the bias and variance of RBK-U can be controlled. Previous works such as \cite{derezinski2024fine,derezinski2024solving,derezinski2025} have characterized the performance of RBK-U for consistent systems that have been preprocessed with a randomized Hadamard transform. After preprocessing the resulting system is highly unlikely to contain nearly singular blocks or other problematic coherence properties, leading to strong convergence guarantees for RBK-U. 
However, it is not clear why matrices that have not been preprocessed would resemble such preprocessed matrices. 
Thus, instead of directly extending these results to the inconsistent case, we instead consider inconsistent problems whose rows arise from a Gaussian distribution, which seem more likely to appear in practice. 
Furthermore, to simplify the analysis we consider the infinite version of the problem in which each uniformly selected row is drawn directly from the corresponding Gaussian distribution. 
This enables us to exploit an elegant connection between uniform block samples of Gaussian data and dense Gaussian sketches of arbitrary data.
It should be possible to extend our results to provide high probability bounds for finite problems using more involved techniques such as those of \cite{derezinski2024fine,derezinski2024solving,derezinski2025}, but we consider this beyond the scope of the current work.
}

Now, consider the statistical least-squares problem 
\begin{equation} \label{eq:stat_ls}
x^* = \argmin{x \in \Rea^n} \expect{[a^\top  \;  b] \sim \mathcal{N}(0,Q)}{(a^\top x - b)^2}.
\end{equation}
where $Q \in \mathbb{R}^{(n+1) \times (n+1)}$ is a positive semidefinite covariance matrix. Furthermore, let $Q_n$ be the top left $n \times n$ block of $Q$, assume for simplicity that $Q_n$ is full rank, and let $Q_n = L_n L_n^\top$ be its Cholesky decomposition. Denote the singular values of $L_n$ by $\sigma_1 \ge \dots \geq \sigma_{n}>0$. 
\begin{theorem}
\label{thm:rbk_gaussian}
\revv{Consider the equivalent of the RBK-U algorithm for the statistical least-squares problem \cref{eq:stat_ls}, for which each sample $A_{S_t}, b_{S_t}$ is drawn by generating $k$ independent samples from $\mathcal{N}(0,Q)$}. Then the results of \cref{thm:rbk} apply \revv{and it holds $x^{(\rho)} = x^*$}.  In addition, the convergence parameter $\alpha$ satisfies
\begin{equation} \label{eq:alpha-bound}
\revv{\alpha} \geq C_{n,k} \max\left\{ \frac{k\sigma_n^2}{\|L_n\|_F^2},  \max_{2\leq \ell < k} \frac{(\ell-1)\sigma_n^2}{ \sum_{i \geq k-\ell-1} \sigma_i^2} \right \}
\end{equation}
with $C_{n,k} \rightarrow 1$ as $n \rightarrow \infty$ for fixed $k$. Furthermore, as long as $k \geq 6$ and $\operatorname{rank}(L_n) \geq 2k$, the variance term satisfies
\begin{equation} \label{eq:variance-gauss}
\revv{V} \leq \frac{200}{\sigma_{2k}^2} \cdot \expect{[a^\top  \;  b] \sim \mathcal{N}(0,Q)}{(a^\top x^* - b)^2}.
\end{equation}
\end{theorem}
In summary, in the case of Gaussian data the ordinary least-squares solution is recovered, the convergence rate $\alpha$ improves at least linearly with the block size $k$, and the variance of the iterates is bounded.  In addition, the following corollary, which is based on Corollary 3.4 of \cite{derezinski2024sharp}, shows that RBK converges much faster  than mSGD for polynomially decaying singular values.

\begin{corollary} \label{coro:gauss}
    Consider the setting of \cref{thm:rbk_gaussian} with fixed \(k \leq n/2\), and assume the \(L_n\) factor has polynomial spectral decay $\sigma_i^2 \leq i^{-\beta} \sigma_1^2$  for all $i$ and some $\beta > 1$. Then the convergence parameters of RBK and mSGD satisfy
   \begin{equation}
    \alpha^{\rm RBK} \geq C k^\beta \tfrac{\sigma_n^2}{\|L_n\|_F^2}, \; 
    \alpha^{\rm mSGD} \leq k \tfrac{ \sigma_n^2}{\|L_n\|_F^2}
    \end{equation}    
    for some constant \(C =  C(\beta) > 0\).
\end{corollary}

Similarly, the faster convergence rate of RBK over mSGD extends to exponentially decaying singular values.
The proofs of \cref{thm:rbk_gaussian,coro:gauss}, provided in \cref{app:rbk}, rely on a connection between RBK-U for Gaussian data and block Gaussian Kaczmarz for arbitrary data. These results generalize the techniques of  \cite{derezinski2024sharp} to the case of inconsistent linear systems, improving upon the results of  \cite{rebrova2021block} in terms of both the convergence rate and the variance.

\subsection{Numerical Demonstration\label{sec:rbk_num}}
We verify our theoretical results for the RBK-U and mSGD algorithms on two problems with Gaussian data, with results in \cref{fig:gaussian_results}. We apply tail averaging to each algorithm to observe convergence beyond the variance horizon. As expected, tail-averaged RBK-U (TA-RBK-U) converges much more rapidly than tail-averaged mSGD (TA-mSGD) in the presence of fast singular value decay.  \rev{We additionally explore the effect of choosing different burn-in periods for TA-RBK-U, with results in \cref{fig:tburns}. As expected, optimal performance is attained by setting $T_b$ to be just slightly after the iterates reach their fixed convergence horizon}. More details on these experiments can be found in \Cref{app:synth}.

\begin{figure}[htbp]
\centering
{\includegraphics[width=0.4\textwidth]{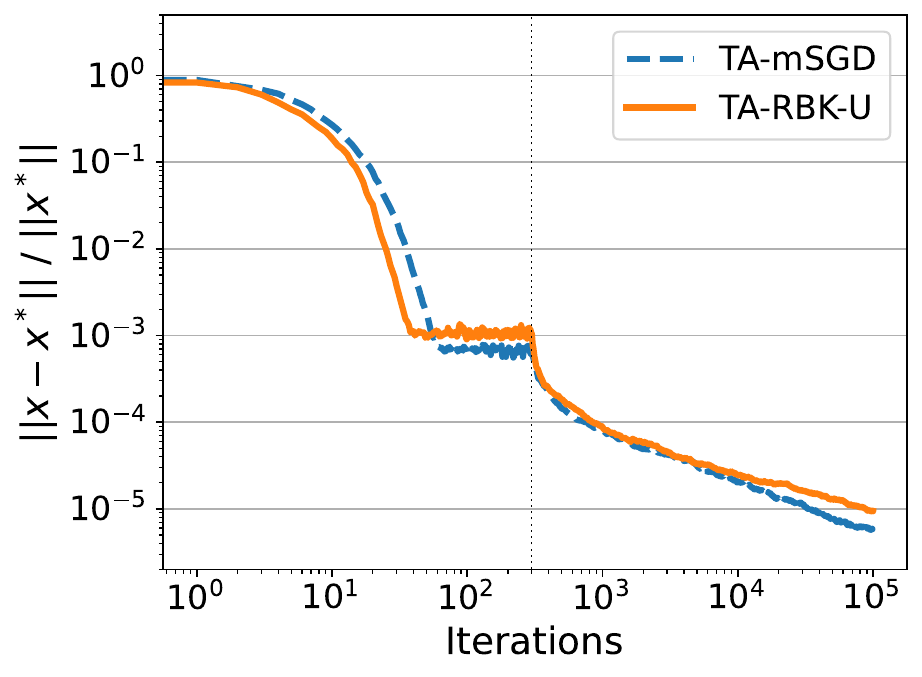}}
\quad
{\includegraphics[width=0.4\textwidth]{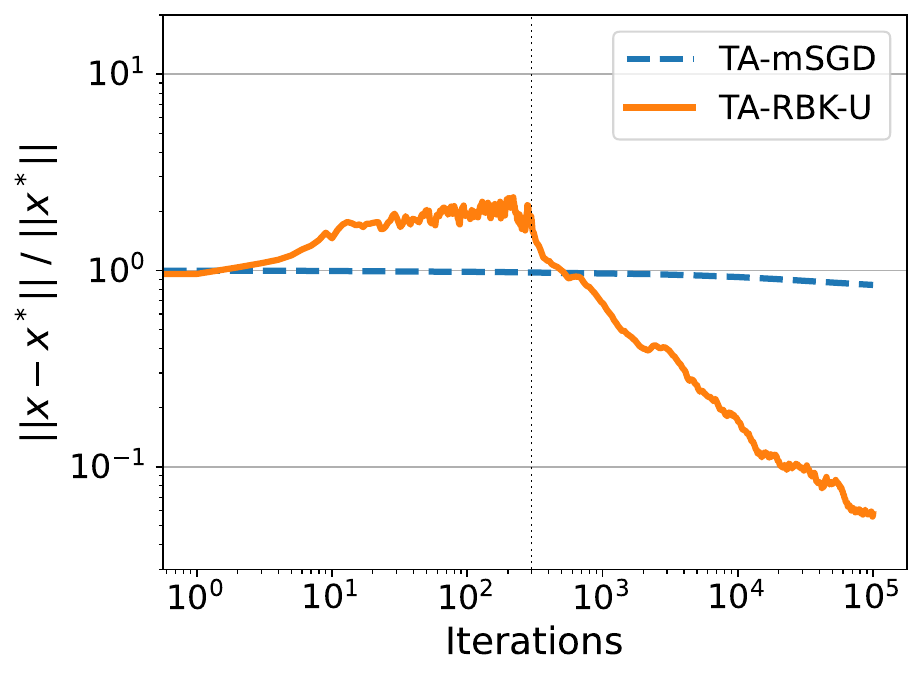}}
\caption{Comparison of methods on two problems with Gaussian data, with no singular value decay (left) and rapid singular value decay (right). The vertical dotted line indicates the burn-in time, before which results are shown for individual iterates.}
\label{fig:gaussian_results}
\end{figure}

\begin{figure}[htbp]
\centering
{\includegraphics[width=0.4\textwidth]{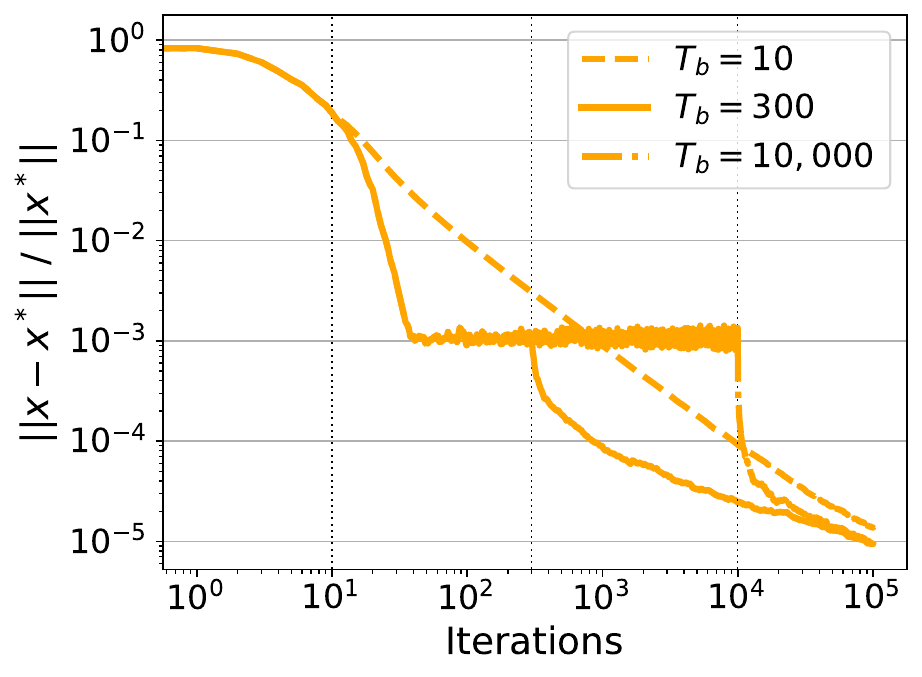}}
\caption{Behavior of TA-RBK-U with different burn-in periods $T_b$. The problem is the same as in the left panel of \cref{fig:gaussian_results}.}
\label{fig:tburns}
\end{figure}

\section{Robustness through Regularization\label{sec:reblock}}
To address the shortcomings of RBK-U in the case of general data, we propose to incorporate regularization into the RBK algorithm. A natural way to do this is to replace the RBK iteration with a stochastic proximal point iteration \cite{asi2019stochastic}, namely
\begin{equation}
x_{t+1} = \argmin{x \in \Rea^n} \left[\norm{A_{S_t} x - b_{S_t}}^2 + \lambda k \norm{x - x_t}^2 \right].
\end{equation}
This minimization problem leads to the closed form
\begin{equation}
x_{t+1} = x_t + A_{S_t}^\top (A_{S_t} A_{S_t}^\top + \lambda k I)^{-1} (b_{S_t} - A_{S_t} x_t),
\end{equation}
and that the RBK iteration \eqref{eq:rbk} is recovered in the limit $\lambda \rightarrow 0$. To incorporate a mild regularization without significantly slowing down the convergence, we propose to use a small, constant value of $\lambda$ throughout the algorithm. We suggest $\lambda=0.001$ as a practical default value, but to obtain optimal performance the value will need to be tuned on a case-by-case basis; see \cref{fig:lambdas} for some numerical results for different values of $\lambda$. We refer to the resulting scheme as the regularized block Kaczmarz method, or ReBlocK.

\begin{theorem}
\label{thm:reblock}
Consider the ReBlocK-U algorithm, namely \cref{alg:gen} with $M(A_S) = (A_S A_S^\top + \lambda k I)^{-1}$ and $\rho = \mathbf{U}(m,k)$. Let  $\alpha = \sigma_{\rm min}^+(\overline{P})$ and assume $x_0 \in \range(A^\top)$. Then the expectation of the ReBlocK iterates $x_T$ converges to $\xrh$ as
\begin{equation}
\label{eq:reblock_exp}
\norm{\expect{}{x_T} - \xrh} \leq (1 - \alpha)^T \norm{x_0 - \xrh}.
\end{equation}
Furthermore, the tail averages $\overline{x}_T$ converge to $\xrh$ as 
\begin{equation}
\label{eq:reblock_converge}
 \expectE \norm{\overline{x}_T - \xrh}^2 \leq 2 (1 - \alpha)^{T_b+1} \norm{x_0 - \xrh}^2 
 + \revv{\frac{4}{\alpha^2 (T-T_b)} V}.
\end{equation}
\revv{with $V = \expectE_{S \sim \rho} \|A_S^\top (A_S A_S^\top + \lambda k I)^{-1} \rrh_S\|^2$. Finally, the condition number $\kappa(\overline{W})$, the weighted residual $r^{(\rho)}$, the bias $x^{(\rho)} - x^*$, and the variance $V$ are bounded as
\begin{equation}
\begin{aligned}
\kappa(\overline{W}) &\leq 1 + \frac{1}{\lambda} \cdot  \max_i \|a_i\|^2, \\
\|r^{(\rho)}\| &\leq \sqrt{\kappa(\overline{W})} \cdot \norm{r}, \\ 
\|x^{(\rho)} - x^*\| &\leq \sqrt{\kappa(\overline{W}) - 1} \cdot \norm{A^+} \cdot \| r\| \leq  \frac{1}{\sqrt{\lambda}} \cdot \kappa(A) \cdot \|r\|, \\
V & \leq \frac{1}{4 \lambda} \cdot \kappa(\overline{W}) \cdot \frac{\norm{r}^2}{m}. 
\end{aligned}
\label{eq:reblock_bounds}
\end{equation}}
\end{theorem}
Note that the values of $\xrh$, $\alpha$, \revv{and $V$} here differ from those in \Cref{sec:rbk} due to the different choice of $M(A_S)$ used by ReBlocK. The advantage relative to RBK-U is that ReBlocK-U is able to control \revv{the condition number of the weight matrix, the weighted residual, the bias of the weighted least-squares solution, and the variance of the iterates} in terms of the reciprocal of the regularization parameter $\lambda$. As a result, the algorithm converges robustly even when the problem contains many nearly singular blocks $A_S$. 

\revv{As a concrete example of the benefits of regularization, recall the isosceles triangle problem \cref{eq:isosceles} for which the bias and variance of RBK-U both diverge as $\epsilon \rightarrow 0$. In contrast, for ReBlocK-U we have the following corollary:
\begin{corollary}
Consider the isosceles triangle problem \cref{eq:isosceles} from \Cref{thm:no-go}. For ReBlocK-U with any fixed $\lambda >0$, the bounds \cref{eq:reblock_bounds} imply 
\begin{equation}
\lim_{\epsilon \rightarrow 0^+} \|r^{(\rho)}\| = \lim_{\epsilon \rightarrow 0^+} \|x^{(\rho)} - x^*\| = \lim_{\epsilon \rightarrow 0^+} V = 0
\end{equation}
since the row norms $a_i$ remain bounded and the residual norm $\norm{r}$ vanishes as $\epsilon \rightarrow 0$.
\end{corollary}
In other words, introducing even a small amount of regularization prevents the algorithm from failing catastrophically. We confirm this effect numerically by directly calculating the bias for various values of $\epsilon$ and $\lambda$, with results in \Cref{fig:thetas}. Note that for fixed $\epsilon$, the bias decreases monotonically with $\lambda$, as expected.}

\begin{figure}[htbp]
    \centering
    \includegraphics[width=0.4\linewidth]{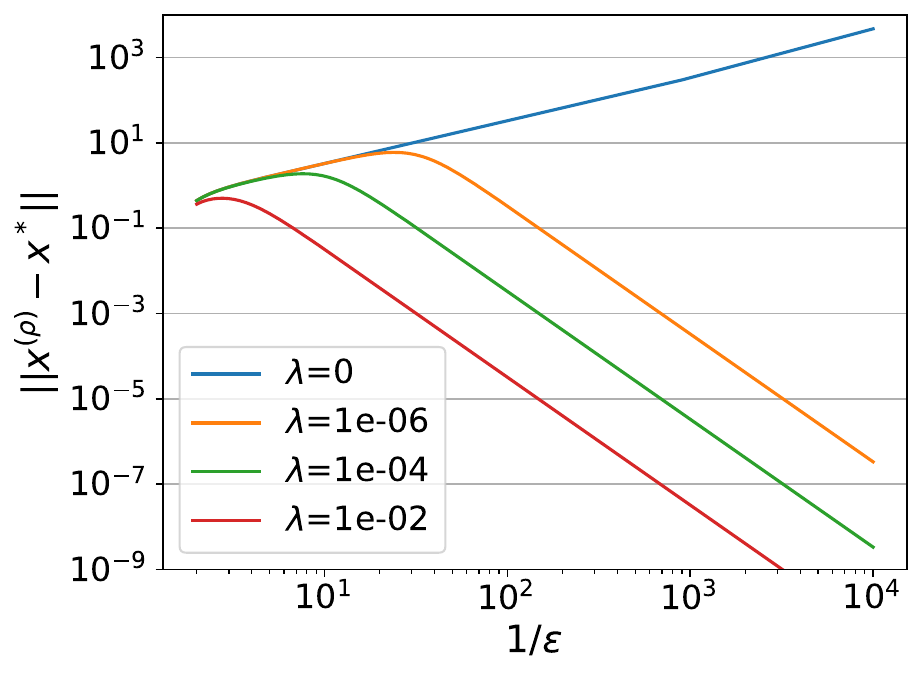} \quad
    \includegraphics[width=0.4\linewidth]{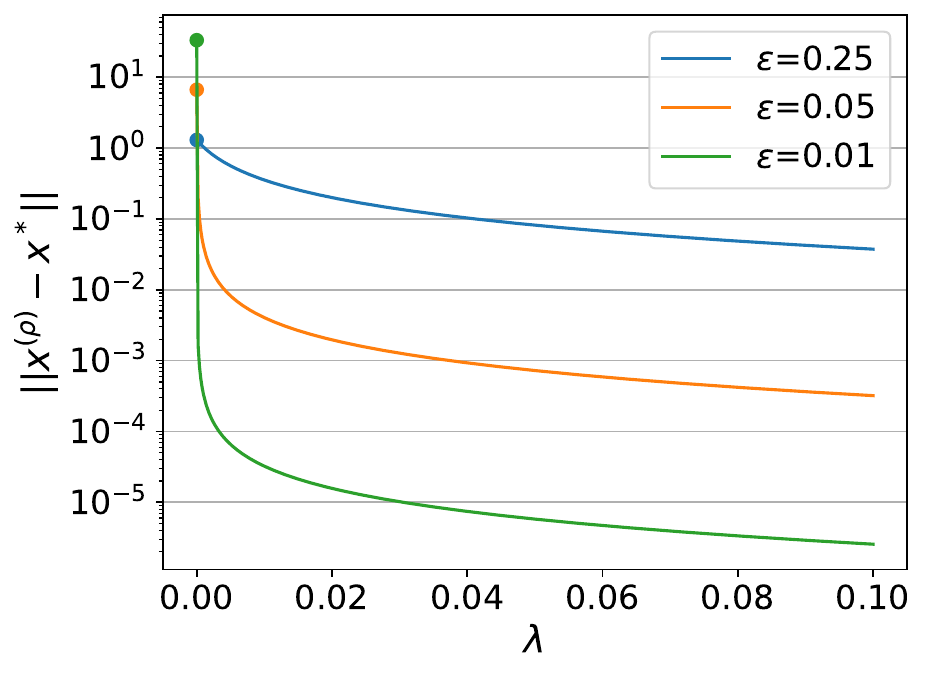}
    \caption{\revv{Size of the bias $x^{(\rho)} - x^*$ for ReBlocK-U with $k=2$ on the isosceles triangle problem of \cref{thm:no-go}, for various values of $\epsilon$ and $\lambda$. Left: for $\lambda=0$ the bias grows without bound as $1/\epsilon \rightarrow \infty$, whereas for any $\lambda > 0$ the bias reaches a fixed maximum value and then decays to zero. Right: for any fixed $\epsilon$ the bias decreases monotonically with $\lambda$.}}
    \label{fig:thetas}
\end{figure}
It is worth noting that these advantages could also be attained by truncating the small singular values in the RBK iteration \eqref{eq:rbk}. However, the ReBlocK iteration is additionally justified by the fact that it is cheaper to implement than the RBK iteration; see \Cref{sec:reblock_imp} and \cref{fig:it_speed} for further discussion of this point. The proof of \cref{thm:reblock} is provided in \cref{app:reblock}. Relative to the proof of \cref{thm:rbk}, the proof of \cref{thm:reblock} is more complicated because the ReBlocK iteration does not lead to an orthogonal decomposition of the error term $x_{t+1} - \xrh$. Instead, the proof relies on a bias-variance decomposition inspired by \cite{defossez2015averaged,jain2018parallelizing,epperly2024randomizedkaczmarztailaveraging}, which also leads to the extra factors of $2$ in the convergence bound. 

We are not yet able to analyze the convergence rate parameter $\alpha$ of ReBlocK-U, even in the case of Gaussian data, as to our knowledge there is no existing work bounding the quality of a regularized Gaussian sketch. However, a fast rate of convergence to the ordinary least-squares solution can be shown when sampling from an appropriate determinantal point process; see \cref{app:dpp}.
Additionally, convergence to the ordinary least-squares solution is obtained for a broader class of \textit{noisy} linear least-squares problems; see \cref{app:noisy}. 

\subsection{Numerical Demonstration\label{sec:reblock_num}}
\revv{We now demonstrate that the benefits of regularization can translate to more realistic types of problems. To do so we revisit the examples from \Cref{fig:rbku_fail_examples}, for which RBK-U exhibited severe instabilities. We observe in \Cref{fig:reblock_success} that ReBlocK-U with $\lambda=1e-3$ is stable for both of these problems} and converges much faster than TA-mSGD in the presence of rapid singular value decay.
To understand the impact of the inconsistency on the algorithms, we re-run the case with rapid singular value decay with varying levels of inconsistency, with results in \Cref{fig:inconsistency}. 
We find qualitatively similar results in every case. Further details on these experiments can be found in \Cref{app:synth}.

\begin{figure}[htbp]
\centering
{\includegraphics[width=0.4\textwidth]{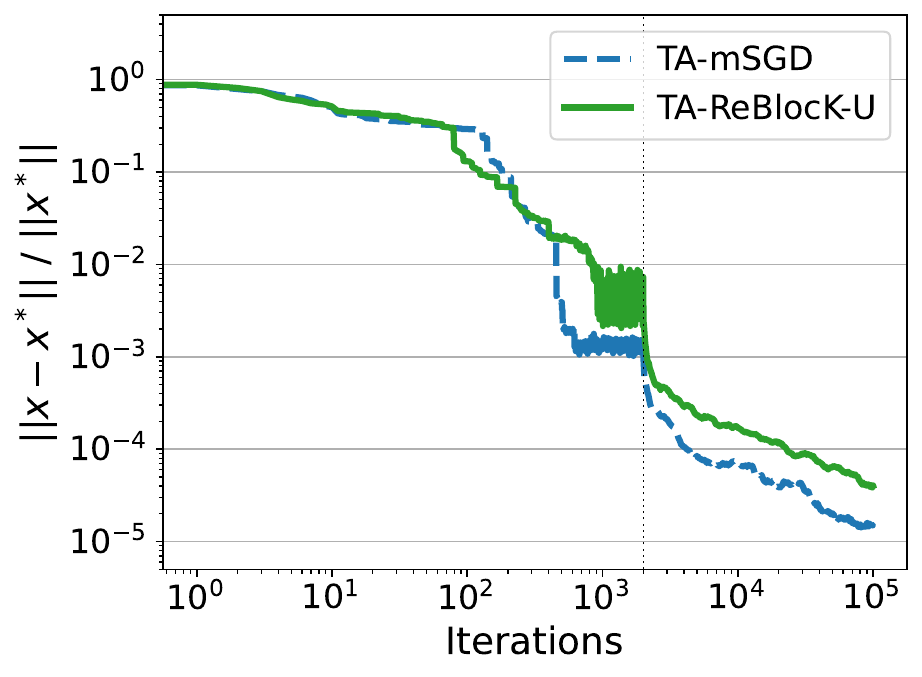}}
\quad
{\includegraphics[width=0.4\textwidth]{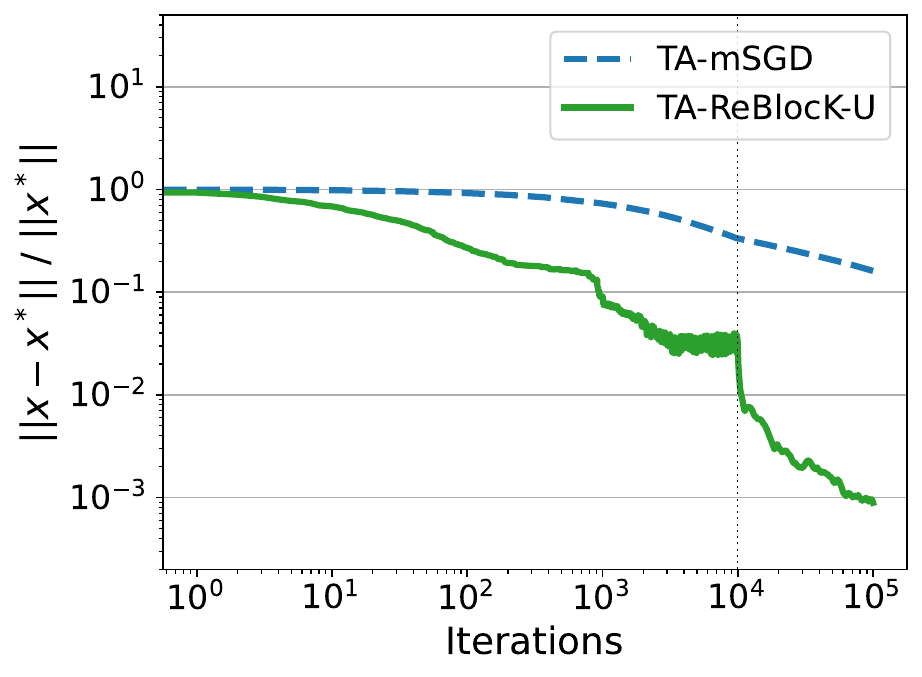}}
\caption{\revv{Performance of ReBlocK-U for the same two inconsistent problems from \Cref{fig:rbku_fail_examples}. These problems both contain} many nearly singular blocks, with mild singular value decay (left) and rapid singular value decay (right). The vertical dotted line indicates the burn-in time, before which results are shown for individual iterates.}
\label{fig:reblock_success}
\end{figure}

\begin{figure}[htbp]
\centering
{
\includegraphics[width=0.4\textwidth]{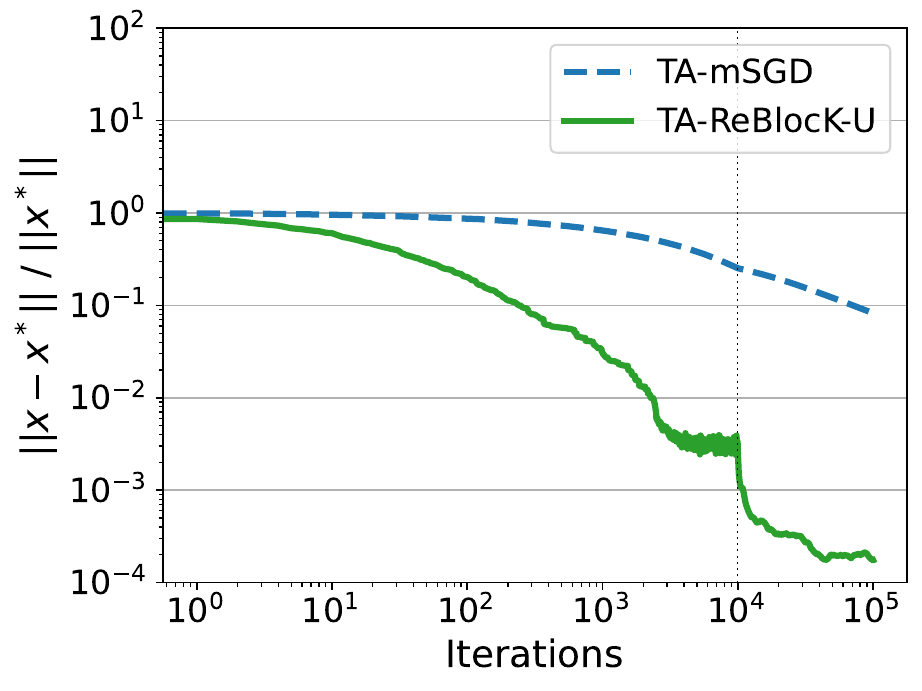} \quad \includegraphics[width=0.4\textwidth]{reblock_succeed_fast.pdf} \\
\includegraphics[width=0.4\textwidth]{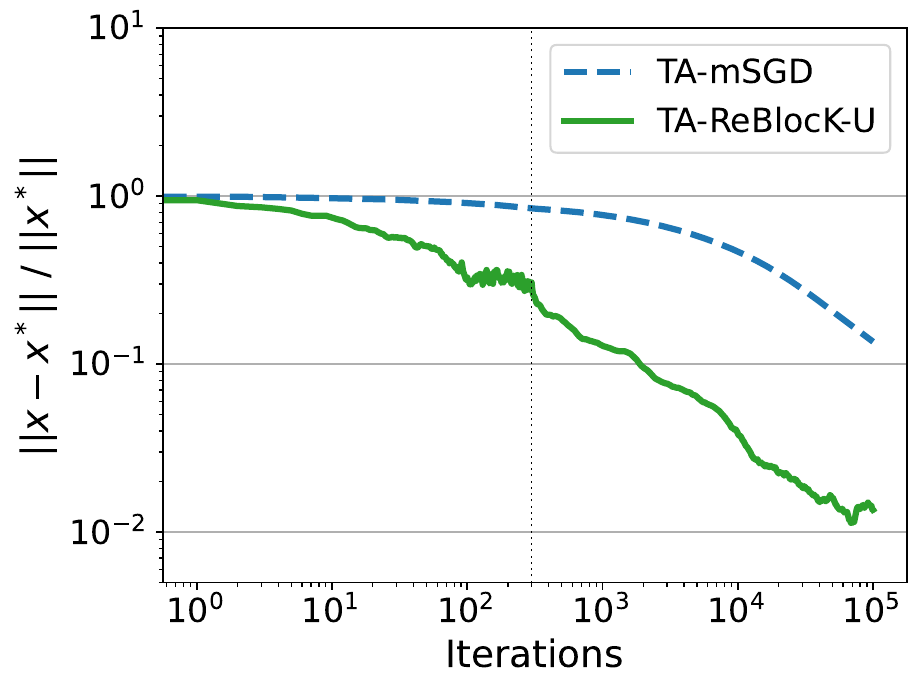}
}
\caption{Exploration of the effect of the inconsistency on the behavior of the algorithms. The variance of each entry of the noise vector is set to $1e-6$ (top left), $1e-4$ (top right), and $1e-2$ (bottom). The top right panel here is identical to the right panel of \Cref{fig:reblock_success}, and is reproduced here only for comparison with the top left and bottom}
\label{fig:inconsistency}
\end{figure}

\subsection{Implementation Details\label{sec:reblock_imp}}
To implement the ReBlocK iterations \eqref{eq:reblock_it}, we directly calculate the $k \times k$ matrix $A_{S_t} A_{S_t}^\top + \lambda k I$ and then use a Cholesky-based linear system solver to calculate $(A_{S_t} A_{S_t}^\top + \lambda k I)^{-1} (b_{S_t} - A_{S_t} x)$. This computation is stable as long as $\lambda$ is not chosen to be too small. Using this approach, the most expensive part of the ReBlocK iteration is calculating $A_{S_t} A_{S_t}^\top$, which has an asymptotic cost of $O(nk^2)$ just like RBK. Nonetheless, in practice the matrix-matrix multiplication for ReBlocK \revv{can have a much smaller preconstant than the QR decomposition} used for RBK. For example, in the experiments of \Cref{sec:ngd}, ReBlocK iterations are over twenty-five times faster than RBK iterations, as reported in \cref{fig:it_speed}.  \revv{We emphasize that this gap is not related to the use of the Cholesky decomposition but instead arises because ReBlocK avoids doing any matrix decomposition at all directly on the $k \times n$ matrix $A_S$.} For the largest problems, the ReBlocK iterations could be further accelerated using iterative solvers; see for example Section 4.1 of \cite{derezinski2025}.

\section{Natural Gradient Optimization\label{sec:ngd}}
\newcommand{\dd}{{\rm d}}

Our original motivation for this work comes from the problem of training deep neural networks using natural gradient descent \cite{amari1998natural}, which is based on an efficient natural gradient induced by a problem-dependent Riemannian metric. 
Natural gradient descent has been studied extensively in the machine learning community; see \cite{martens2015optimizing, ren2019efficient, martens2020new}. Furthermore, there is increasing evidence that natural gradient methods can improve the accuracy when training neural networks to solve physical equations. See \cite{muller2023achieving,dangel2024kronecker} for applications to physics-informed neural networks and \cite{pfau2020ab,schatzle2023deepqmc} for applications to neural network wavefunctions. 

To elucidate the structure of the natural gradient direction, consider the function learning problem
\begin{equation}
\label{eq:func_learn}
\min_\theta L(\theta),\quad
L(\theta) \coloneqq \frac{1}{2} \int_\Omega  (f_\theta(s) - f(s))^2 \dd s,
\end{equation}
where $\Omega \subset \Rea^d$ is the domain of the functions, $f:\Omega \rightarrow \Rea$ is the target function and $f_\theta: \Omega \rightarrow \Rea$ is a function represented by a neural network with parameters $\theta \in \Rea^n$. The standard definition of natural gradient descent for this problem is
\begin{equation}
\theta \gets \theta - \eta G_N,\quad G_N \coloneqq F^{-1} \nabla_\theta L(\theta), \label{eq:ngd_pre}
\end{equation}
where $\eta$ is the step size, $F$ is the Fisher information matrix
\begin{equation}
F = \int_\Omega \nabla_\theta f_\theta(s)  \nabla_\theta f_\theta(s)^\top \dd s = J^\top J,
\end{equation}
and the Euclidean gradient $\nabla_\theta L(\theta)$ takes the form
\begin{equation}
\nabla_\theta L(\theta) = \int_\Omega  \nabla_\theta f_\theta(s)  (f_\theta(s) - f(s)) \dd s = J^\top [f_\theta -f].
\end{equation}
Here $J: \Rea^n \rightarrow \Rea^\Omega$ represents the Jacobian, which is a linear operator from the space of parameters to the space of real-valued functions on $\Omega$. $J^\top$ represents the adjoint.

Calculating  the natural gradient direction $G_N$ using \eqref{eq:ngd_pre} requires a linear solve against the $n \times n$ matrix $F$, which is very challenging in realistic settings when $n \geq 10^6.$ This has motivated the development of approximate schemes such as \cite{martens2015optimizing}. An alternative approach is to reformulate $G_N$ using the structure of $F$ and $\nabla_\theta L(\theta)$:
\begin{align}
G_N &= (J^\top J)^{-1} J^\top [f_\theta -f] \\
&= \argmin{x \in \Rea^n} \norm{J x - [f_\theta - f]}^2, \label{eq:ngd_ls}
\end{align}
where the norm in the final expression is the $L_2$-norm in function space. This least-squares formulation has been pointed out for example by \cite{martens2020new, chen2024empowering,goldshlager2024kaczmarz}, with the work of Chen and Heyl empowering major advances in the field of neural quantum states. A major goal of the current work is to provide a more solid foundation for the development of natural gradient approximations along these lines.

The natural way to access the data when training a neural network is to sample a set of points in the domain $\Omega$ and evaluate the target function $f$, the network outputs $f_\theta$, and the network gradients $\nabla_\theta f_\theta$ at the sampled points. This is precisely equivalent to sampling a small subset of the rows of \eqref{eq:ngd_ls}, which motivates the consideration of row-access least-squares solvers for calculating natural gradient directions. Furthermore, both empirical evidence from scientific applications \cite{park2020geometry,wang2022and} and theoretical evidence from the literature on neural tangent kernels \cite{bietti2019inductive,ronen2019convergence,cao2019towards} suggest that $J$ should be expected to exhibit fast singular value decay, motivating the possibility that RBK and ReBlocK could converge rapidly when solving \eqref{eq:ngd_ls}. Note that these observations translate straightforwardly from the simple function learning problem \eqref{eq:func_learn} to the realistic problems of training physics-informed neural networks or neural network wavefunctions.

\revv{It is also worth noting that while we focus on the problem of calculating natural gradient directions, similar linear systems also arise when using neural networks to simulate the time evolution of either differential equations \cite{bruna2024neural} or quantum systems \cite{schmitt2020quantum}. When simulating time dynamics it is essential to solve each linear system to high accuracy. This provides additional motivation for developing block row access methods that can solve such linear systems accurately without preprocessing.  }
\subsection{Numerical Demonstration}

\begin{figure}[htbp]
\centering
\centering
{\includegraphics[width=0.4\textwidth]{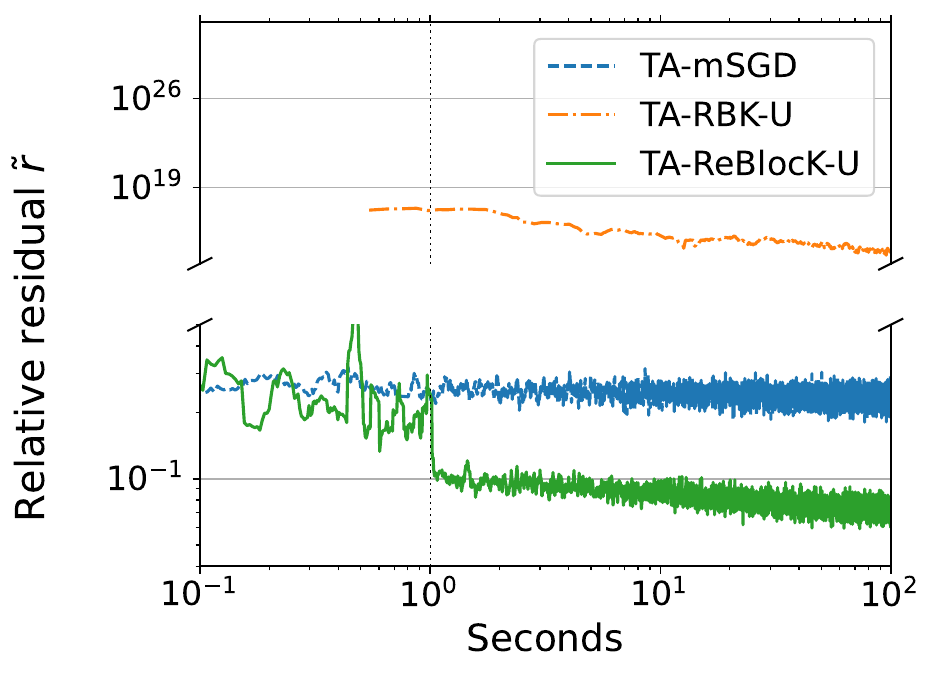}}
{\includegraphics[width=0.4\textwidth]{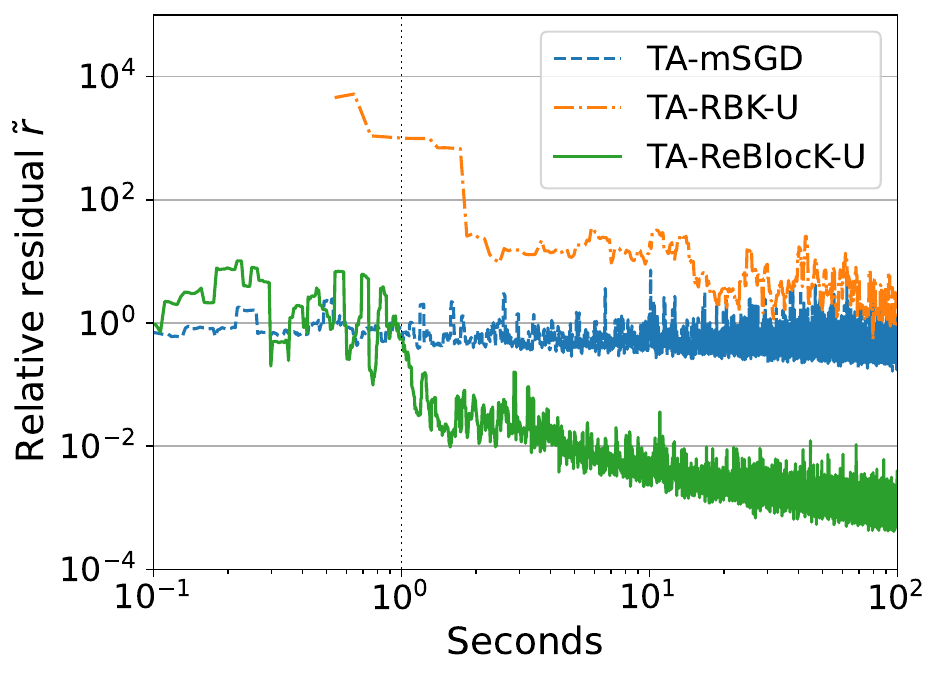}}\\
{\includegraphics[width=0.4\textwidth]{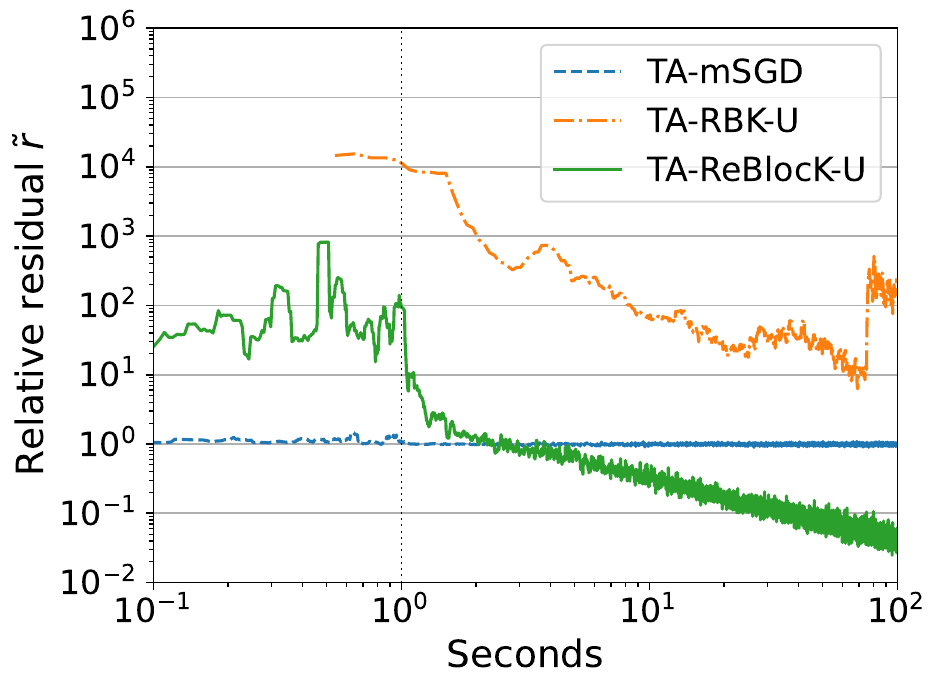}}
\caption{Comparison of methods for calculating natural gradient directions for a small neural network. The network parameters $\theta$ are taken from three snapshots of a single training run, with one snapshot from the ``pre-descent'' phase before the loss begins to decrease (top left), one snapshot from the ``descent'' phase during which the loss decreases rapidly (top right), and one snapshot from the ``post-descent'' phase when the decay rate of the loss has slowed significantly (bottom). The algorithms are measured in terms of their progress towards reducing the relative residual $\tilde{r} = \norm{J x - [f_\theta - f]} / \norm{f_\theta - f}$ for the least-squares problem \eqref{eq:ngd_ls}, which measures how well the function-space update direction $J x$ agrees with the function-space loss gradient $f_\theta - f$. The burn-in time is set to $T_b \approx T/100$ in each case, as indicated by the vertical dotted line.}
\label{fig:nn}
\end{figure}

\begin{figure}[htbp]
\centering
\centering
{\includegraphics[width=0.4\textwidth]{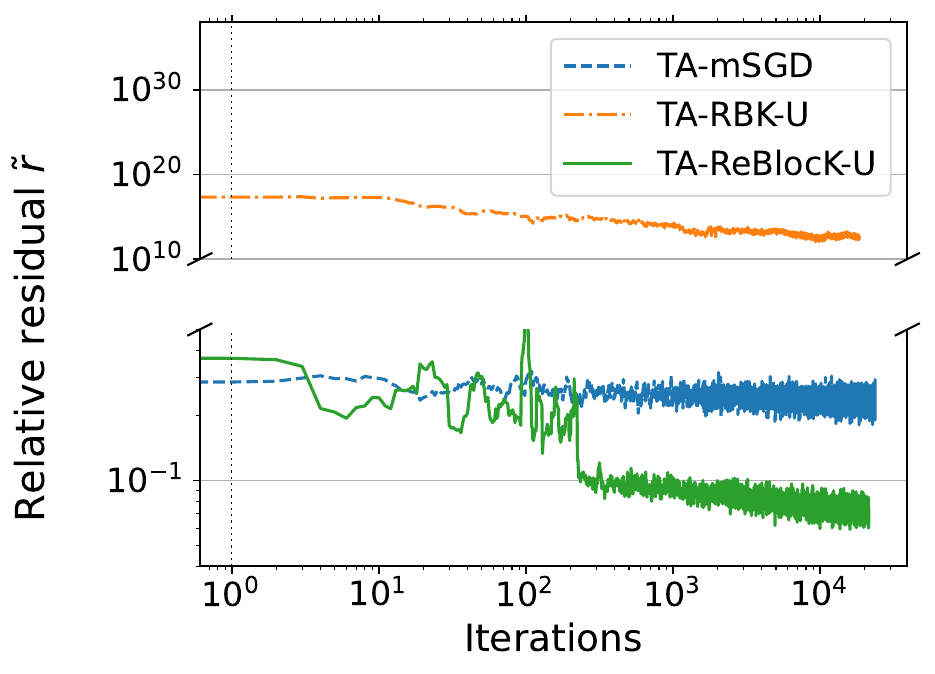}}
{\includegraphics[width=0.4\textwidth]{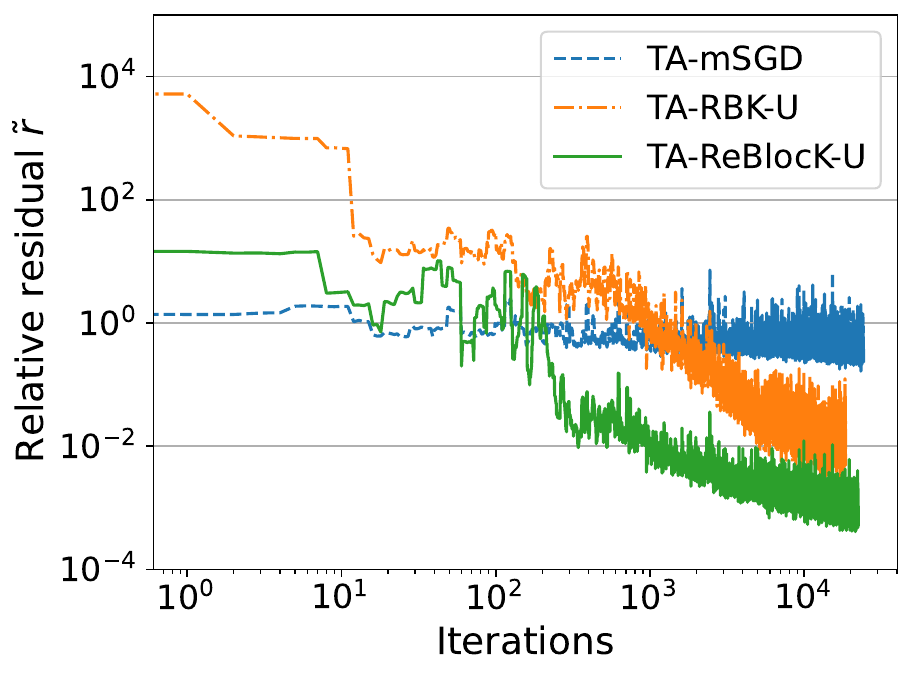}} \\
{\includegraphics[width=0.4\textwidth]{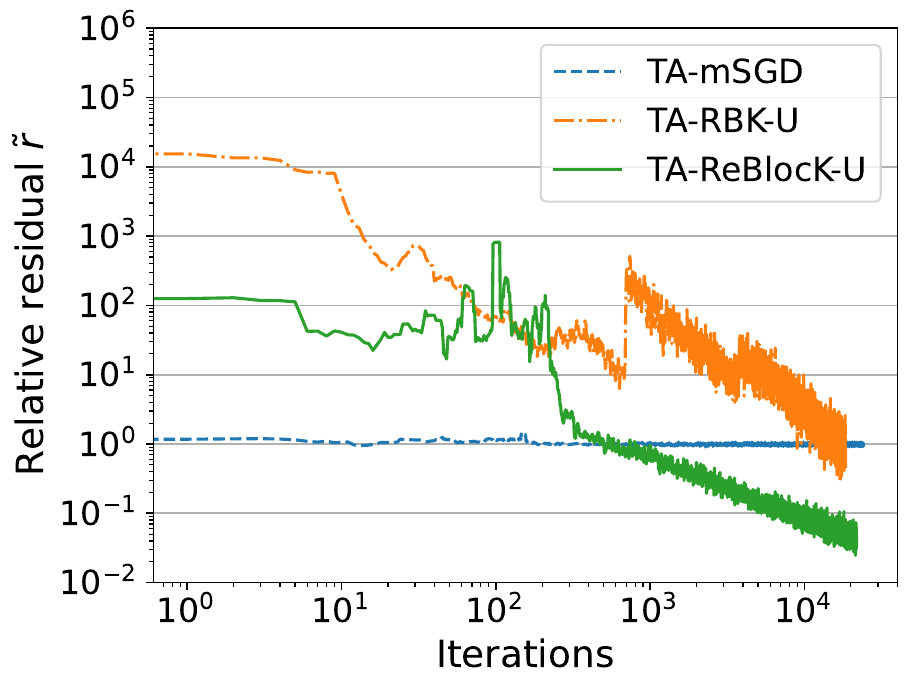}} 
\caption{\revv{Per iteration plots for the same three snapshots from \Cref{fig:nn}, with TA-RBK-U run for much longer than the other methods to produce a comparable number of iterations. The algorithms are measured in terms of their progress towards reducing the relative residual $\tilde{r} = \norm{J x - [f_\theta - f]} / \norm{f_\theta - f}$ for the least-squares problem \eqref{eq:ngd_ls}, which measures how well the function-space update direction $J x$ agrees with the function-space loss gradient $f_\theta - f$. The burn-in period is set based on seconds rather than iterations and so is not indicated in this plot.}}
\label{fig:per_it}
\end{figure}

To test our algorithms for calculating natural gradient directions, we train a neural network with about $7.5 \times 10^5$ parameters to learn a rapidly oscillating function on the unit interval. We take three snapshots from the training process and form the least-squares problem \eqref{eq:ngd_ls} for each.  The methods TA-mSGD, TA-RBK-U, and TA-ReBlocK-U are then compared on these problems with results in \cref{fig:nn}. The computations are performed on an A100 GPU to simulate a deep learning setting \revv{and in double precision to enable each algorithm's iterations to be computed accurately}, and progress is measured with wall-clock time on the horizontal axis. TA-ReBlocK-U performs best in every case. \revv{Interestingly, RBK-U fails catastrophically on the first snapshot and also exhibits instabilities in the last snapshot, providing further evidence that the issues highlighted by the no-go \Cref{thm:no-go} are not merely of theoretical interest. We also provide per iteration convergence plots in \Cref{fig:per_it}. In the first snapshot the catastrophic failure of TA-RBK-U persists while in the second and third snapshots, we find that TA-RBK-U is more competitive on a per iteration basis, but is still outperformed by TA-ReBlocK-U.}

In \cref{fig:lambdas}, we investigate the effect of the parameter $\lambda$ on the performance of ReBlocK. \revv{We observe that choosing $\lambda=1$ already provides an advantage relative to mSGD, and the advantage is further improved by reducing $\lambda$ down to $1e-3$ or $1e-6$. However, when choosing $\lambda=1e-9$ the variance becomes so large as to significantly hinder the performance of the algorithm}. These results suggest that the performance of ReBlocK is optimized by choosing $\lambda$ as small as possible without introducing significant instabilities. Next, in \cref{fig:ks} we explore the effect of the batch size $k$ on the performance of the algorithms. We observe that ReBlocK performs best across all batch sizes. Finally, in \cref{fig:it_speed}, we report the number of iterations per second for mSGD, RBK, and ReBlocK for the experiments of this section. We find that the ReBlocK iterations are over twenty-five times faster than the RBK iterations and only about \revv{$10\%$} slower than the mSGD iterations, \revv{which provides a significant boost to TA-ReBlocK-U even beyond its improved stability}. Further details on these experiments can be found in \Cref{app:nn}. 

\begin{figure}[htbp]
    \centering
    \includegraphics[width=0.45\textwidth]{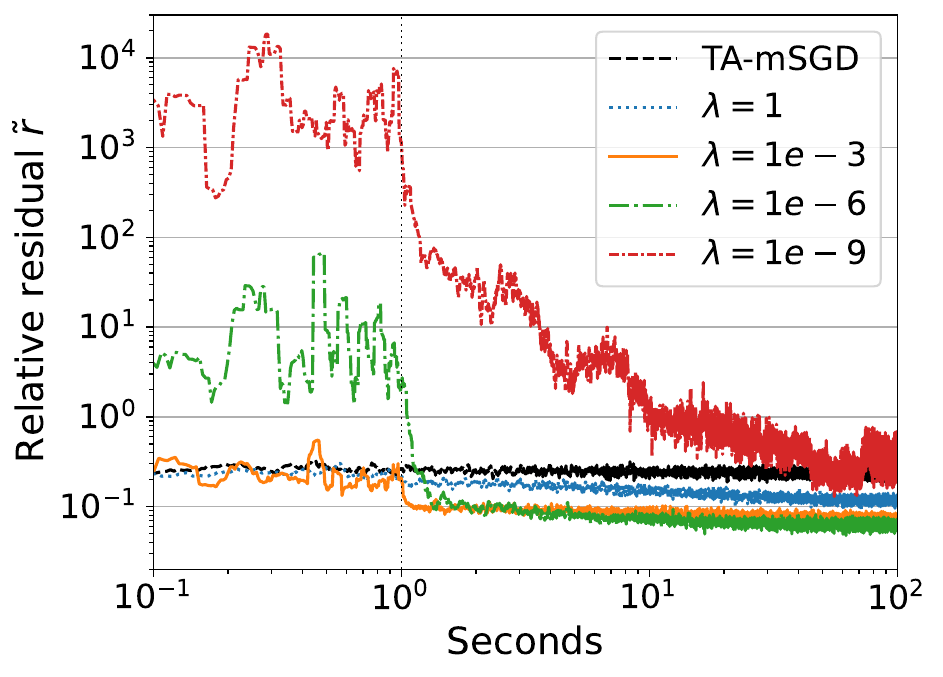}
    \caption{\rev{Performance of TA-ReBlocK-U on the problem from the top left panel of \cref{fig:nn} with various values of $\lambda$.}}
\label{fig:lambdas}
\end{figure}

\begin{figure}[htbp]
\centering
\centering
{\includegraphics[width=0.4\textwidth]{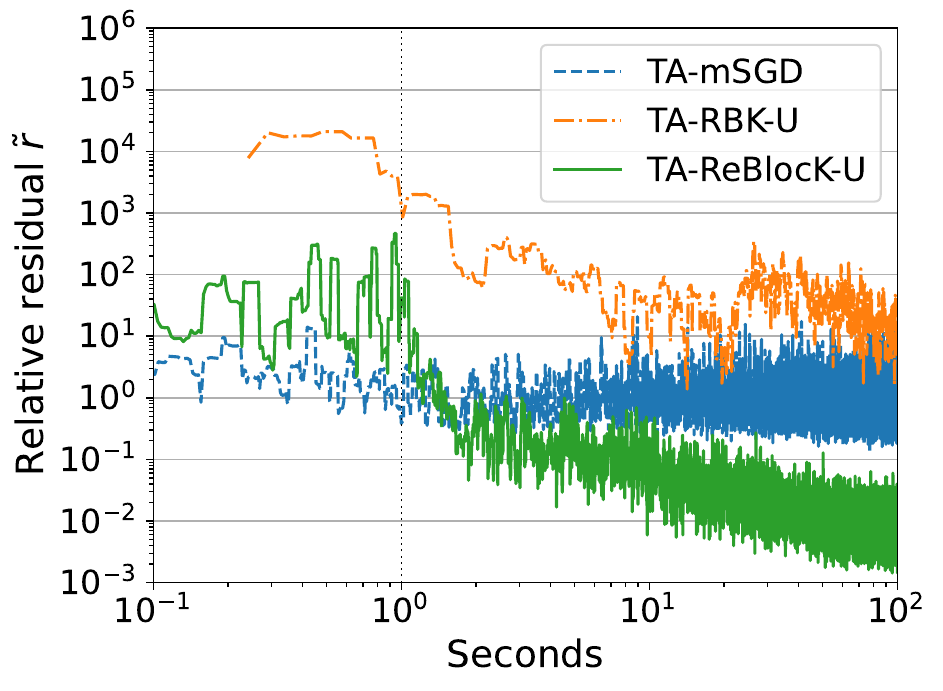}}
{\includegraphics[width=0.4\textwidth]{NN_mid.pdf}} \\
{\includegraphics[width=0.4\textwidth]{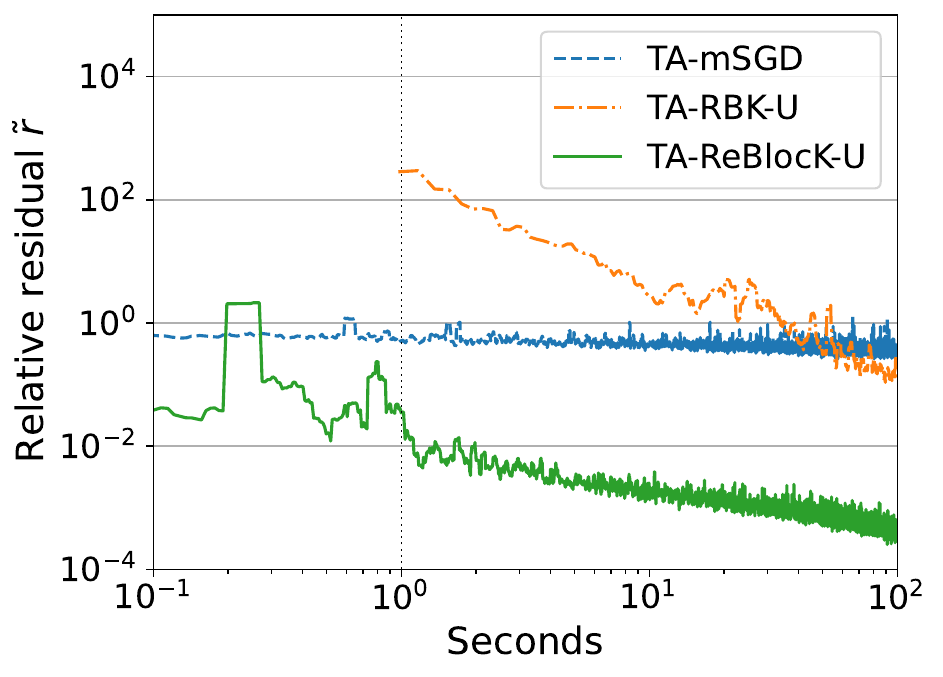}}
{\includegraphics[width=0.4\textwidth]{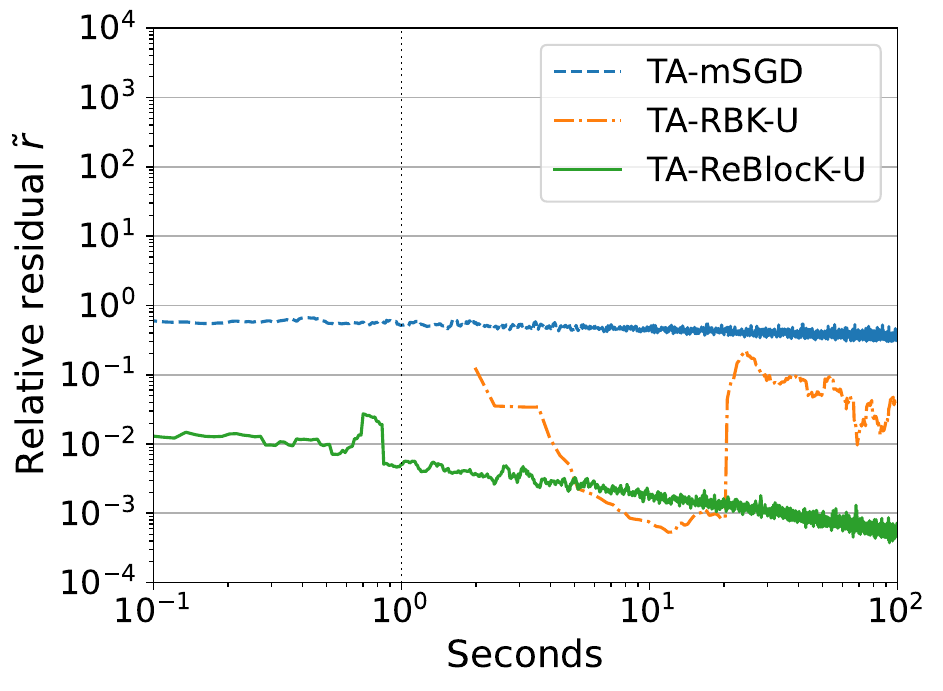}}
\caption{\rev{Performance of algorithms for different values of the block size $k$. The problem is the same as in the top right panel of \cref{fig:nn} and the values of $k$ are $k=25$ (top left), $k=50$ (top right), $k=100$ (bottom left), and $k=200$ (bottom right).  The top right panel here is identical to the top right panel of \cref{fig:reblock_success}, and is reproduced here only for comparison with the other quadrants.}}
\label{fig:ks}
\end{figure}

\begin{figure}[htbp]
\centering
\includegraphics[width=0.4\textwidth]{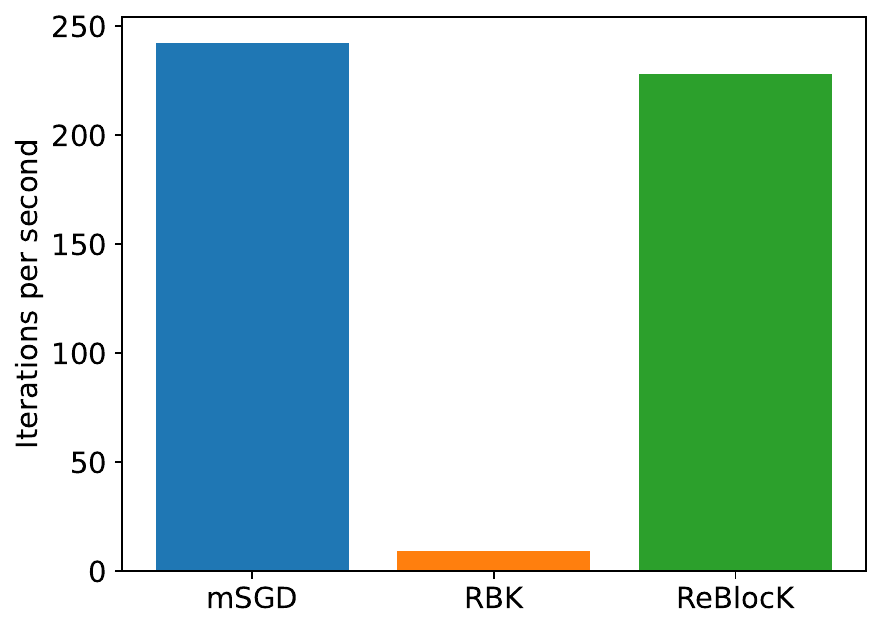} 
\caption{Iteration speed of each algorithm for the natural gradient experiments with $k=50$.}
\label{fig:it_speed}
\end{figure}
 
Our results, though preliminary due to their synthetic nature, suggest that ReBlocK is a promising method for calculating natural gradient directions. This provides justification for the Kaczmarz-inspired SPRING algorithm \cite{goldshlager2024kaczmarz} and suggests a family of related methods to be explored by future works. Open questions include how many ReBlocK iterations to run between each update of $\theta$, whether to incorporate averaging, and how to tune $\lambda$.

\section{Conclusions}
In this work, we have explored the problem of solving large-scale linear least-squares problems without preprocessing the data matrix $A$. Our results suggest that ReBlocK is a more effective algorithm than either RBK or mSGD for inconsistent problems that exhibit rapid singular value decay. More broadly, our work highlights the value of incorporating regularization as a path towards broadening the applicability of block Kaczmarz methods, and our analysis suggests that a Monte Carlo perspective can be useful in elucidating the behavior of randomized block row-access methods in general. Finally, our work provides motivation and suggests new directions for \revv{both natural gradient optimizers and time stepping methods in scientific machine learning.}

\section*{Acknowledgements}
We thank Ethan Epperly, Robert Webber, Marius Zeinhofer, Ruizhe Zhang, and the anonymous reviewers for their thoughtful discussions and feedback.

\appendix

\section{Helpful Lemmas}

In this appendix we prove five lemmas which will help us analyze both RBK and ReBlocK. \cref{lemma:expect} shows how the expectation of the iterates evolves. \cref{lemma:range} shows that the iterates of the algorithms always stay within $\range(A^\top)$, and \cref{lemma:P} provides a contraction property of applying $I - \overline{P}$ to certain types of vectors. \cref{lemma:tail} shows how to derive bounds on the expected square error of tail-averaged iterates under relevant assumptions. Finally, \cref{lemma:res} bounds the \revv{weighted residual and the bias} based on the condition number of $\overline{W}$.

\begin{lemma}[Convergence of single-iterate expectation]
\label{lemma:expect}
Consider the generalized iteration \eqref{eq:it_general} with some fixed choice of a positive semidefinite  mass matrix $M(A_S)$ and sampling distribution $\rho$, and fix two indices $r < s$. Then the expectation of $x_s$ conditioned on $x_r$ satisfies
\begin{equation}
\expect{}{x_s - \xrh  | x_r} = (I - \overline{P})^{s-r} (x_r - \xrh).
\end{equation}
\end{lemma}
\begin{proof}[Proof of \cref{lemma:expect}]
Fix $t \in r, \ldots, s-1$, let $P_t = P(S_t)$ and $W_t = W(S_t)$. Recall from \eqref{eq:it_nice} that the generalized iteration \eqref{eq:it_general} can be reformulated as
\begin{equation}
x_{t+1} - \xrh = (I - P_t) (x_t - \xrh) + A^\top W_t \rrh.
\end{equation}
Recall also that the normal equations for $\xrh$ imply that $A^\top \overline{W} \rrh = 0$. Taking the expectation over the choice of $S_t$, it thus obtains
\begin{equation}
\expect{}{x_{t+1} - \xrh | x_t}= (I - \overline{P}) (x_t - \xrh) + A^\top \overline{W} \rrh = (I - \overline{P}) (x_t - \xrh).
\end{equation}
Using the law of total expectation and the linearity of expectation, iterating this result for $t = r, \ldots, s-1$ yields the lemma.
\end{proof}

\begin{lemma}
\label{lemma:range}
Consider the generalized iteration \eqref{eq:it_general} with some fixed choice of a positive semidefinite mass matrix $M(A_S)$ and sampling distribution $\rho$. Then $x^*,\xrh \in \range(A^\top)$ and if $x_0 \in \range(A^\top)$, then $x_t \in \range(A^\top)$ for all $t \geq 0$.
\end{lemma}
\begin{proof}
Recall that $x^*$ is the minimal-norm solution to \eqref{eq:ls} and $\xrh$ is the minimal-norm solution to \eqref{eq:xrh}. Suppose for the sake of contradiction that $x^* \notin \range(A^\top)$. Then let $\tilde{x}$ be the projection of $x^*$ onto $\range(A^\top)$. It follows that $\norm{\tilde{x}} < \norm{x^*}$ and $Ax^* = A\tilde{x}$, making $\tilde{x}$ a minimizer of \eqref{eq:ls} with a smaller norm than $x^*$ (contradiction). The same argument holds for $\xrh$. 

To show $x_t \in \range(A^\top)$ for all $t \geq 0$ we apply an inductive argument. The claim holds for $t=0$ by assumption, so now suppose that it holds for some arbitrary $t \geq 0$. Recall
\begin{equation}
 x_{t+1} = x_t + A_{S_t}^\top M(A_{S_t}) (b_{S_t} - A_{S_t} x_t).
 \end{equation}
 Then $x_t \in \range(A^\top)$ by assumption and  $A_{S_t}^\top M(A_{S_t}) (b_{S_t} - A_{S_t} x_t) \in \range(A^\top)$ since $\range(A_{S_t}^\top )\subseteq \range(A^\top)$. It follows that $x_{t+1} \in \range(A^\top)$, completing the proof.

\end{proof}

\begin{lemma}
\label{lemma:P}
Suppose that $x \in \range(\overline{P})$, $\overline{P} \preceq I$, and $\alpha = \sigma_{\rm min}^+(\overline{P})$. Then for any $s \geq 0$, 
\begin{equation}
x^\top (I - \overline{P})^s  x \leq (1 - \alpha)^s \norm{x}^2
\end{equation}
and 
\begin{equation}
\norm{(I - \overline{P})^s x} \leq (1 - \alpha)^s \norm{x}.
\end{equation}
\end{lemma}
\begin{proof}
First expand $x$ in the basis of eigenvectors of the symmetric matrix $\overline{P}$, then note that every eigenvector that has a nonzero coefficient in the expansion has its eigenvalue in the interval $[\alpha, 1]$. 

\end{proof}
\begin{lemma}[Convergence of tail-averaged schemes] 
\label{lemma:tail}
Consider the generalized iteration \eqref{eq:it_general} with some fixed mass matrix $M(A_S)$ and sampling distribution $\rho$. Let $\alpha = \sigma_{\rm min}^+(\overline{P})$ and suppose that $x_0 \in \range(A^\top)$, $\range(\overline{P}) = \range(A^\top)$, and $\overline{P} \preceq I$. Additionally, suppose the stochastic iterates $x_0, x_1, \ldots$ satisfy 
\begin{equation}
\expectE \norm{x_t- \xrh}^2 \leq (1 - \alpha)^t B + \tilde{V}
\end{equation}
for all $t$ and some constants $B, \tilde{V}$, \revv{where we use $\tilde{V}$ to distinguish from the variance term $V$ in \Cref{thm:rbk,thm:reblock}}. Then the tail averages $\overline{x}_T$ of the iterates, with burn-in time $T_b$, satisfy
\begin{equation}
\expectE \norm{\overline{x}_T - \xrh}^2 \leq (1 - \alpha)^{T_b+1} B + \frac{2}{\alpha (T - T_b)} \tilde{V}.
\end{equation}
\end{lemma}
\begin{proof}
We follow closely the Proof of Theorems 1.2 and 1.3 of \cite{epperly2024randomizedkaczmarztailaveraging}. Decompose the expected mean square error as
\begin{equation}
\label{eq:mse_sum}
\expectE \norm{\overline{x}_T - \xrh}^2 = \frac{1}{(T-T_b)^2} \sum_{r,s = T_b+1}^{T} \expect{}{(x_r - \xrh)^\top (x_s - \xrh)}.
\end{equation}
For $T_b \leq r < s$ bound the covariance term using \cref{lemma:expect} and our assumptions on $\overline{P}$: 
\begin{align*}
\expect{}{(x_r - \xrh)^\top (x_s - \xrh)} &= \expect{}{(x_r - \xrh)^\top \expect{}{x_s - \xrh | x_r}} \\
&= \expect{}{(x_r - \xrh)^\top (I - \overline{P})^{s-r} (x_r - \xrh)} \\
& \leq (1-\alpha)^{s-r} \expectE \norm{x_r - \xrh}^2 \\
& \leq  (1 - \alpha)^s B + (1-\alpha)^{s-r} \tilde{V},
\end{align*}
where the third line uses \cref{lemma:range,lemma:P} and the assumptions $\range{\overline{P}} = \range(A^\top)$ and $\overline{P} \preceq I$.

Using the coarse bound $  (1 - \alpha)^s B \leq (1 - \alpha)^{T_b+1} B$ for $s \geq T_b+1$, it follows
\begin{equation}
 \expectE \norm{\overline{x}_T - \xrh}^2 \leq (1 - \alpha)^{T_b+1} B + \frac{\tilde{V}}{(T-T_b)^2}  \sum_{r,s = T_b+1}^T (1-\alpha)^{|s-r|}.
\end{equation}
Apply another coarse bound
\begin{equation}
\sum_{r,s = T_b+1}^T (1-\alpha)^{|s-r|} \leq 2 \sum_{r = T_b+1}^{T} \sum_{s=0}^\infty (1-\alpha)^s = \frac{2(T - T_b)}{\alpha}
\end{equation}
to obtain the final result:
\begin{equation}
 \expectE \norm{\overline{x}_T - \xrh}^2 \leq  (1 - \alpha)^{T_b+1} B + \frac{2}{\alpha (T-T_b)} \tilde{V}.
\end{equation}
\end{proof}

\begin{lemma}[Residual \revv{and bias bounds}]
\label{lemma:res}
Assuming $\overline{W} \succ 0$, it holds
\begin{equation}
\begin{aligned}
\norm{\rrh}^2 &\leq \kappa(\overline{W}) \cdot \norm{r}^2, \\
\revv{ \norm{x^{(\rho)} - x^*}} & \revv{\leq \sqrt{\kappa(\overline{W}) - 1} \cdot \norm{A^+} \cdot \norm{r}.}
\end{aligned}
\end{equation}
\end{lemma}
\begin{proof}
\revv{To bound the weighted residual}, use operator norms and the optimality of $\xrh$ for the weighted least-squares problem: 
\begin{align*}
\norm{A \xrh - b}^2 & \leq \norm{\overline{W}^{-1}} \norm{A \xrh - b}_{\overline{W}}^2 \\
&\leq \norm{\overline{W}^{-1}} \norm{A x^* - b}_{\overline{W}}^2 \\
& \leq \norm{W^{-1}} \norm{\overline{W}}  \norm{r}^2 \\
& = \kappa(\overline{W})  \norm{r}^2 .
\end{align*}
\revv{Next, note that the weighted residual admits an orthogonal decomposition  $r^{(\rho)} = A(x^* - x^{(\rho)}) + r$. It follows that
\begin{equation}
\norm{A (x^{(\rho)} - x^*)}^2 = \norm{r^{(\rho)}}^2 - \norm{r}^2 \leq (\kappa(\overline{W}) - 1) \cdot \norm{r}^2.
\end{equation}
Taking square-roots and using $\norm{x^{(\rho)} - x^*} \leq \norm{A^+} \cdot \norm{A(x^{(\rho)} - x^*)}$ 
yields the bias bound
\begin{equation}
\norm{x^{(\rho)} - x^*} \leq \sqrt{\kappa(\overline{W}) - 1} \cdot \norm{A^+} \cdot \norm{r}.
\end{equation}}
\end{proof}

\newcommand{\zt}{z^{(t)}}
\section{Proofs of RBK Convergence Theorems\label{app:rbk}}
In this section we provide the proofs of \cref{thm:rbk,thm:rbk_gaussian,coro:gauss}.
We first prove the following lemma, which can be viewed as a more general version of Theorem 1.2 of \cite{needell2014paved}:
\begin{lemma}
\label{lemma:rbk}
Consider the RBK iterates \eqref{eq:rbk} and fix some sampling distribution $\rho$. Suppose that $x_0 \in \range(A^\top)$ and $\range(\overline{P}) = \range(A^\top)$. Then for all $t$ it holds
\begin{equation}
\expectE \norm{x_t - \xrh}^2 \leq (1 - \alpha)^t \norm{x_0 - \xrh}^2 + \frac{1}{\alpha} \expectE_{S \sim \rho} \norm{A_S^+ \rrh_S}^2. 
\end{equation}\end{lemma}
\begin{proof}
Let $P_s = P(S_s)$ and $W_s = W(S_s)$. Using \eqref{eq:it_nice} and the definition $M(A_S) = (A_S A_S^\top)^+$, the RBK iteration \eqref{eq:rbk} can be reformulated as 
\begin{equation}
x_{s+1} - \xrh = (I - P_s) (x_s - \xrh) + A_{S_s}^+ \rrh.
\end{equation}
This represents an orthogonal decomposition of $x_{s+1} - \xrh$ since $I-P_s$ is the projector onto the null space of $A_{S_s}$, which is in turn orthogonal to the range of the pseudoinverse $A_{S_s}^+$. 

It thus holds
\begin{align*}
\norm{x_{s+1} - \xrh}^2 &= \norm{(I - P_s) (x_s - \xrh)}^2 + \norm{A_{S_s}^+ \rrh_{S_s}}^2 \\
&= (x_s - \xrh)^\top (I - P_s) (x_s - \xrh)  + \norm{A_{S_s}^+ \rrh_{S_s}}^2,
\end{align*}
where the second line uses the idempotency of the projector $I - P_s$. Note that $x_s - \xrh \in \range(A^\top)$ by \cref{lemma:range} and $\overline{P} \preceq I$ since $P(S)$ is always a projection matrix. Taking expectations and applying \cref{lemma:P} thus yields
\begin{align*}
\expectE \norm{x_{s+1} - \xrh}^2 &= (x_s - \xrh)^\top (I - \overline{P}) (x_s - \xrh)  + \expectE_{S \sim \rho} \norm{A_S^+ \rrh_S}^2 \\
& \leq (1-\alpha) \expectE \norm{x_s - \xrh}^2 + \expectE_{S \sim \rho} \norm{A_S^+ \rrh_S}^2.
\end{align*}
Iterating from $s=0$ to $s=t-1$ and utilizing $\sum_{s=0}^{t-1} (1-\alpha)^s < \sum_{s=0}^\infty (1-\alpha)^s = \sfrac{1}{\alpha}$ yields the desired result.
\end{proof}
\begin{proof}[Proof of \cref{thm:rbk}]
The result follows from the appropriate application of \cref{lemma:expect,lemma:tail,lemma:rbk}. To apply these lemmas it is first required to verify $\overline{P} \preceq I$ and $\range(\overline{P}) = \range(A^\top)$. For the first result, it suffices to observe that $P(S)$ is the orthogonal projector onto the row space of $A_S$ and thus $P(S) \preceq I$ for all $S$. 

For the second result, first note that since $\overline{P} = A^\top \overline{W} A$ it is clear that $\range(\overline{P}) \subseteq \range(A^\top)$. Utilizing this containment and the fact that $\overline{P}$ is symmetric, it suffices to show that there is no $x \in \range(A^\top)$ such that $\overline{P} x = 0$. Now, consider any such $x \in \range(A^\top)$. There must exist a row index $i$ such that $a_i^\top x \neq 0$. Thus
\begin{equation}
x^\top \overline{P} x = \expect{S \sim \mathbf{U}(m,k)}{x^\top A_S^\top  (A_S A_S^\top)^{-1} A_S x} \geq \frac{k}{m} \cdot \frac{(a_i^\top x)^2}{\norm{a_i}^2} > 0,
\end{equation}
where the intermediate step uses the facts that row $i$ is chosen with probability $k/m$ and that $A_S^\top (A_S A_S^\top)^{-1} A_S \succeq a_i (a_i^\top a_i)^{-1} a_i^\top$ when $i \in S$.  It follows that $\overline{P} x \neq 0$ and so indeed $\range(\overline{P}) = \range(A^\top)$.

With the assumptions verified, \cref{lemma:expect} can be applied using $r=0$, $s=T$ to yield
\begin{equation}
\expect{}{x_T} - \xrh = (I - \overline{P})^T (x_0 - \xrh).
\end{equation}
This directly implies
\begin{equation}
\norm{\expect{}{x_T} - \xrh} \leq (1- \alpha)^T \norm{x_0 - \xrh}
\end{equation}
using \cref{lemma:range,lemma:P} and the proven properties of $\overline{P}$.

Furthermore, \cref{lemma:rbk} holds and provides the conditions for \cref{lemma:tail} with $B = \norm{x_0 - \xrh}^2$, $\tilde{V} = \frac{1}{\alpha} \expectE_{S \sim \rho} \norm{A_S^+ \rrh_S}^2$. The application of \cref{lemma:tail} then directly implies the convergence result for the tail-averaged RBK-U algorithm.

\revv{Now, we proceed with upper bounds on condition number, the weighted residual, the bias, and the variance. For the condition number, first note that
\begin{equation} \label{eq:AAT}
\norm{A_S A_S^\top} = \norm{A_S^\top A_S}  = \norm{\sum_{i \in S} a_i a_i^\top} \leq\sum_{i \in S} \norm{ a_i a_i^\top}  \leq k \cdot \revv{\max_i \norm{a_i}^2}.
\end{equation}
Since we have assumed the rows of $A$ are in general position, this implies $\|(A_S A_S^\top)^+\| \geq 1 / (k \cdot \revv{\max_i \norm{a_i}^2})$.
Since $\overline{W} = \expect{S}{I_S^\top (A_S A_S^\top)^+ I_S}$ follows that
\[
\frac{1}{k \cdot \max_i \norm{a_i}^2} \cdot \mathbb{E}[I_S^\top I_S] \preceq \overline{W}\preceq \max_S \|A_S^+\|^2 \cdot \mathbb{E}[I_S^\top I_S].
\]
When $S$ is sampled uniformly, $\expect{S \sim \mathbf{U}(m,k)}{I_S^\top I_S} = \frac{k}{m} I$ and so this implies
\[
\kappa(\overline{W}) \le\ k \cdot \Big(\max_{i}\|a_i\|^2\Big)\cdot \max_S \|A_S^+\|^2 .
\]
}

\revv{
The bounds on the weighted residual and the bias follow directly from \Cref{lemma:res}.
Finally, regarding the variance, it follows from the definition of $V$ that
\begin{equation} 
\begin{aligned}
V & = \expectE_S \|A_S^+ r_S^{(\rho)}\|^2
\le \max_S\|A_S^+\|^2\ \cdot \,\mathbb{E}_{S} \|r_S^{(\rho)}\|^2 \\
& = \max_S\|A_S^+\|^2 \cdot \frac{k \|r^{(\rho)}\|^2}{m}.
\end{aligned}
\end{equation}
Applying the weighted residual bound from  \Cref{lemma:res}  the desired result.
}
\end{proof}

\begin{proof}[Proof of \cref{thm:rbk_gaussian}] 
\revv{In the case of Gaussian data, we consider the statistical linear regression problem
\begin{equation}
x^* = \argmin{x \in \Rea^n}  \expect{[a^\top  \;  b] \sim \mathcal{N}(0,Q)}{(a^\top x - b)^2}.
\end{equation}
The RBK-U algorithm is applied by generating \(k\) independent samples per iteration from the data distribution $\mathcal{N}(0,Q)$. Denote by $Q= LL^\top$ the Cholesky decomposition of $Q$.   
Let $L=\begin{bmatrix} L_a \\ \ell_b^\top \end{bmatrix}$ so $L_a$ represents the first $n$ rows of $L$ and $\ell_b^\top$ represents the last row of $L$. Due to the lower triangular nature of $L$ it follows that $L_a = \begin{bmatrix} L_n & 0 \end{bmatrix}$ where $L_n$ is the $n \times n$ Cholesky factor of $Q_n$ (which we previously defined to be the top left $n \times n$ block of $Q$).}

In each iteration of RBK-U, the $k$ data points $[a_1^\top\; b_1], \dots, [a_k^\top\; b_k]$ from $\mathcal{N}(0, Q)$ can be equivalently redistributed as $z_1^\top [L_a^\top\; \ell_b],\dots, z_{k}^\top [L_a^\top\; \ell_b]$ for $z_i \sim \mathcal{N}(0, I_{n+1})$.
Collecting the random vectors $z_1, \ldots, z_k$ into the columns of a single matrix $Z_t$, the RBK-U update can be written as
\begin{equation} \label{eq:rbk-gauss} x_{t+1} = x_t + (Z_t^\top L_a^\top)^{+}( Z_t^\top \ell_b - Z_t^\top L_a^\top x_t). \end{equation}

Now, note that $\expect{}{aa^\top} = L_a L_a^\top, \expect{}{b a} = L_a \ell_b$ and $\expect{}{b^2} = \ell_b^\top \ell_b$. Plugging these identities into the definition of $x^*$, we find
\begin{equation}
x^* = \argmin{x \in \Rea^n} \expect{[a^\top  \;  b] \sim \mathcal{N}(0,Q)}{(a^\top x - b)^2}  
= \argmin{x \in \Rea^n} \| L_a^\top x - \ell_b \|^2.
\end{equation}
We can thus define an ``underlying'' residual vector by $\tilde{r} = \ell_b - L_a^\top x^*$. 
Additionally defining $P_t = (Z_t^\top L_a^\top)^{+}  Z_t^\top L_a^\top$, the update \eqref{eq:rbk-gauss} can be reformulated as 
\begin{equation} \label{eq:rbk-gauss-fixed} x_{t+1} - x^* = (I - P_t)(x_t - x^*) + ( Z_t^\top L_a^\top)^{+} Z_t^\top  \tilde{r}. \end{equation}

\revv{
Now, decompose $Z_t^\top = [Z_t^1 \; Z_t^2]$ where $Z_t^1$ consists of the first $n$ columns and $Z_t^2$ of the last column. Note that $Z_t^1, Z_t^2$ are independent mean-zero Gaussian matrices and note also that that since $L_a = [L_n \; 0]$ it holds $Z_t^\top L_a^\top = Z_t^1 L_n^\top$. Furthermore since $L_n$ is full rank, $\range(L_a^\top) = {\rm span}(e_1, \ldots, e_n)$ and thus $\tilde{r} \parallel e_{n+1}$ (here $e_i$ is the $i$th standard basis vector in $\Rea^{n+1}$). It follows that $Z_t^\top \tilde{r} = Z_t^2 \tilde{r}_{n+1}$ and so
$$\expect{}{ ( Z_t^\top L_a^\top)^{+}( Z_t^\top \tilde{r})} =  \expect{}{ (Z_t^1 L_n^\top)^+ Z_t^2 \tilde{r}_{n+1}} = \expect{}{ (Z_t^1 L_n^\top)^+}\expect{}{ Z_t^2 \tilde{r}_{n+1}} = 0$$
since $\expect{}{Z_t^2} = 0$}. We note that a similar independence lemma has also been established in Lemma 3.14 of \cite{rebrova2021block}.
Furthermore, taking the expectation on both sides of \eqref{eq:rbk-gauss-fixed}, we have
\[ \expect{}{x_{t+1} - x^*} = (I - \overline{P}) (x_t - x^*),   \]
where $\overline{P}:= \expect{}{P_t}$. 
Consequently, \revv{$\expect{}{x_t}$ converges to $x^*$ using the same logic as the proof of \cref{thm:rbk}, which proves that $x^{(\rho)} = x^*$.}

\rev{Now, consider the bound on the value of $\alpha$.} Note that the entries of 
$Z_t \in \mathbb{R}^{(n + 1) \times k}$ are independently drawn from the standard Gaussian distribution. This property enables us to leverage existing results from \cite{derezinski2024sharp} concerning the spectrum of the matrix $\overline{P}$. Specifically, $P_t = (Z_t L_a^\top)^{+} Z_t L_a^\top$, which has a similar formula to the matrix $P$ considered in Equation (6) of \cite{derezinski2024sharp}. 
\revv{Then, noting that $L_a = [L_n \; 0]$ has the same singular value spectrum as $L_n$}, applying Theorem 3.1 of \cite{derezinski2024sharp} gives the result stated in \cref{thm:rbk_gaussian}. 

\rev{It remains to show the bound on the variance.} Note that the variance term $\norm{A_{S_t}^\top r_{S_t}}^2$ in \cref{thm:rbk} here takes the form
\begin{equation}
\norm{(Z_t^\top L_a^\top)^+ Z_t^\top \tilde{r}}^2.
\end{equation}
We bound
\begin{equation}
\norm{(Z_t^\top L_a^\top)^+ Z_t^\top \tilde{r}}^2 = \norm{(Z_t^1 L_n^\top)^+ Z_t^2 \tilde{r}_{n+1}}^2 \leq \sigma_{\rm min}^{-2}(Z_t^1 L_n^\top) \norm{Z_t^2 \tilde{r}_{n+1}}^2.
\end{equation}

We can bound the expectation of the smallest singular value of $Z_t^1 L_n^\top$ using a similar technique to the proof of Lemma 22 in \cite{derezinski2024sharp}. To start, let $L_n^\top = W\Sigma Y^\top$ be the SVD of $L_n^\top$. Then 
\begin{equation}
\sigma_{\rm min}^2(Z_t^1 L_n^\top) = \sigma_{\rm min}(Z_t^1 L_n^\top L_n (Z_t^1)^\top) = \sigma_{\rm min}(Z_t^1 W \Sigma^2 W^\top (Z_t^1)^\top).
\end{equation}
Let $W_{2k}$ denote the first $2k$ columns of $W$ and note that $\Sigma \succeq {\rm diag}(\sigma_{2k}, \dots, \sigma_{2k}, 0,\dots, 0)$ where we recall that $\sigma_i$ is the $i$th singular value of $L_n$. It thus holds
\[ \sigma_{\rm min}^2(Z_t^1 L_n^\top) \geq \sigma_{2k}^2 \cdot \sigma_{\rm \min}(Z_t^1 W_{2k} W_{2k}^\top (Z_t^1)^\top) = \sigma_{\min}^{2}(Z_t^1 W_{2k}).  \]
Since the columns of $W_{2k}$ are orthonormal the random matrix $Z_t^1 W_{2k}$ can be redistributed as a single Gaussian random matrix $G_{2k}$ of size $k \times (2k)$ and with standard normal entries. By Lemma 3.16 of \cite{rebrova2021block}, we thus have for $k \geq 6$, 
\begin{equation}
\expect{}{\sigma_{\rm min}^{-2}(Z_t^1 L_n^\top)} \leq \frac{20}{(\sqrt{2k} - \sqrt{k})^2 \sigma_{2k}^2} \leq \frac{200}{k \sigma_{2k}^2}.
\end{equation}
Hence, by the independence between $Z_t^1$ and $Z_t^2$ we have
\begin{equation}
\expectE \norm{(Z_t^\top L_a^\top)^+ Z_t^\top \tilde{r}}^2 = \expect{}{\sigma_{\rm min}^{-2}(Z_t^1 L_n^\top)} \expectE \norm{Z_t^2 \tilde{r}_{n+1}}^2 \leq \frac{200}{\sigma_{2k}^2} \cdot \|\tilde{r}\|^2.  
\end{equation}

\revv{Noting that $ \expect{[a^\top  \;  b] \sim \mathcal{N}(0,Q)}{(a^\top x^* - b)^2} = \|L_a^\top x^* - \ell_b\|^2 = \|\tilde{r}\|^2$, we prove \eqref{eq:variance-gauss}}.
\end{proof}

\begin{proof}[Proof of \cref{coro:gauss}]
    By Corollary 3.4 of \cite{derezinski2024sharp}, if $L_n$ has the polynomial spectral decay of order $\beta>1$, i.e., $\sigma_i^2 \leq  i^{-\beta} \sigma_1^2$ for all $i$,
    the dependence of $C_{n,k}$ on $n$ and $k$ in \cref{thm:rbk_gaussian} can be eliminated when $k \leq n/2$.  
    Furthermore, there is a constant $C=C(\beta)$ such that for any $k\le n/2$, the linear convergence rate satisfies   
 \[ \alpha^{\rm RBK} \geq C\frac{k^\beta \sigma_n^2}{\|L_n\|_F^2}. \]

Regarding the convergence rate of mSGD, first consider a finite matrix $A$. It is demonstrated in \cite{needell2014stochastic} that, with a minibatch size of 1 and importance sampling, the corresponding convergence parameter can be at most $\alpha^{\rm SGD} \leq \kappa_{\rm dem}^{-2}(A)$ for convergent step sizes. \revv{It is straightforward to generalize this result to the statistical problem \cref{eq:stat_ls}, where the corresponding convergence bound is $\alpha^{\rm SGD} \leq \kappa_{\rm dem}^{-2}(L_n)$}. When a larger minibatch size \(k\) is used, \cite{jain2018parallelizing} shows that the learning rate can be increased at most linearly, corresponding to a rate of at most 
\[\alpha^{\rm mSGD} \leq k  \kappa_{\rm dem}^{-2}(L_n) =\frac{k\sigma_n^2}{\| L_n\|_F^2}. \]
\end{proof}

\section{Proofs of ReBlocK Convergence Theorems\label{app:reblock}}
This appendix contains the proof of the ReBlocK convergence bound \cref{thm:reblock}. We first establish a useful lemma which bounds the error of the individual ReBlocK iterates under arbitrary sampling distributions.

\begin{lemma}
\label{lemma:reblock}
Consider the ReBlocK iterates \eqref{eq:reblock_it} and fix some sampling distribution $\rho$. Suppose that $x_0 \in \range(A^\top)$ and $\range(\overline{P}) = \range(A^\top)$. Then for all $t$ it holds
\begin{equation}
 \expectE \norm{x_t - \xrh}^2 \leq 2 (1 - \alpha)^{t} \norm{x_0 - \xrh}^2 + \revv{\frac{2}{\alpha} \expectE_{S \sim \rho} \norm{A_S^\top (A_S A_S^\top + \lambda k I)^{-1} \rrh_S}^2.}
\end{equation}
\end{lemma}
\begin{proof}
Since the ReBlocK iteration does not admit a simple orthogonal decomposition, we instead proceed by utilizing a \textit{bias-variance decomposition} inspired by \cite{defossez2015averaged,jain2018parallelizing, epperly2024randomizedkaczmarztailaveraging}.

Let $P_s = P(S_s)$ and $W_s = W(S_s)$. Bias and variance sequences are defined respectively by
\begin{equation}
d_0 = x_0 -\xrh, \; d_{s+1} = (I-P_s) d_s,
\end{equation}
\begin{equation}
v_0 = 0,\;v_{s+1} = (I - P_s) v_s + A^\top W_s \rrh.
\end{equation} Intuitively, the bias sequence captures the error due to the initialization $x_0 \neq \xrh$ and the variance sequence captures the rest of the error. It can be verified by mathematical induction and \eqref{eq:it_nice} that 
$x_s-\xrh = d_s + v_s$ for all $s$. As a result, it holds
\begin{equation}
\label{eq:bv}
\expectE \norm{x_t - \xrh}^2 \leq 2 \expectE \norm{d_t}^2 + 2 \expectE \norm{v_t}^2.
\end{equation}
Note also that it is a simple extension of \cref{lemma:range} that $d_s,v_s \in \range(A^\top)$ for all $s$.

To analyze the bias, first calculate
\begin{equation}
\expect{}{\norm{d_{s+1}^2} | d_s} = d_s^\top \expect{}{(I-P_s)^2} d_s = d_s^\top (I - 2\overline{P} + \expect{S \sim \rho}{P(S)^2}) d_s.
\end{equation}
Observe that
\begin{equation}
P(S) = A_S^\top (A_S A_S^\top + \lambda k I)^{-1} A_S \preceq  A_S^\top (A_S A_S^\top)^{-1} A_S \preceq I,
\end{equation}
and hence $P(S)^2 \preceq P(S)$ and $\expect{S \sim \rho}{P(S)^2} \preceq \overline{P}$. As a result, $I - 2\overline{P} + \expect{S \sim \rho}{P(S)^2} \preceq I - \overline{P}$. Additionally leveraging the properties $d_s \in \range(A^\top) = \range(\overline{P})$ and $\overline{P} \preceq I$, \cref{lemma:P} implies
\begin{equation}
\expect{}{\norm{d_{s+1}^2} | d_s} \leq (1-\alpha) \norm{d_s}^2.
\end{equation}
Iterating this inequality yields a simple bound on the bias term:
\begin{equation}
\label{eq:bias_bound}
\expectE \norm{d_t}^2 \leq (1-\alpha)^t \norm{x_0 - \xrh}^2.
\end{equation}

The analysis of the variance term is less straightforward. To begin, note that $v_s$ follows the same recurrence as $x_s - \xrh$, so we can apply a slight modification of \cref{lemma:expect} to the variance sequence to obtain
\begin{equation}
\expect{}{v_s} = (I - \overline{P})^s v_0 = 0.
\end{equation}

Now, square the iteration for $v_{s+1}$ conditioned on $v_s$:
\begin{align*}
\expect{}{\norm{v_{s+1}}^2 | v_s} &= v_s^\top \expect{S \sim \rho}{(I-P(S))^2} v_t + 2 v_s^\top \expect{S \sim \rho}{(I -P(S))A^\top W(S) \rrh} + \expectE_{S \sim \rho} \norm{A^\top W(S) \rrh}^2 \\
& \leq (1-\alpha) \norm{v_s}^2 + 2 v_s^\top \expect{S \sim \rho}{(I -P(S))A^\top W(S) \rrh} + \expectE_{S \sim \rho} \norm{A^\top W(S) \rrh}^2,
\end{align*}
where the second step uses our previous observations that $\expect{S \sim \rho}{(I - P(S))^2} \preceq I - \overline{P}$. $v_s \in \range(\overline{P})$, and \cref{lemma:P}. Take expectations over $v_s$ as well to eliminate the cross-term, yielding
\begin{align*}
\expectE \norm{v_{s+1}}^2 \leq (1-\alpha) \expectE \norm{v_s}^2 + \expectE_{S \sim \rho} \norm{A^\top W(S) \rrh}^2.
\end{align*}

Iterating this last inequality for $s = 0, \ldots, t-1$, using $\sum_{s=0}^{t-1} (1-\alpha)^s < \sum_{s=0}^\infty (1-\alpha)^s = \sfrac{1}{\alpha}$, and recalling that $v_0 = 0$, we find
\begin{equation}
\expectE \norm{v_{t}}^2  \leq \frac{1}{\alpha} \expectE_{S \sim \rho} \norm{A^\top W(S) \rrh}^2. \label{eq:reblock_var}
\end{equation}

Noting that 
\begin{equation}
A^\top W(S) \rrh = A_S^\top (A_S A_S^\top + \lambda k I)^{-1} \rrh_S,
\end{equation}
\revv{we can combine the bias bound \eqref{eq:bias_bound} and the variance bound \eqref{eq:reblock_var} using \eqref{eq:bv} to obtain the desired result.}
\end{proof}

\begin{proof}[Proof of \cref{thm:reblock}]
The result will follow from the appropriate application of \cref{lemma:expect,lemma:tail,lemma:reblock}. To apply these lemmas it is first required to verify $\overline{P} \preceq I$ and $\range(\overline{P}) = \range(A^\top)$. For the first result, it suffices to observe that $P(S) \preceq P_{RBK}(S) \preceq I$ for all $S$. 

For the second result, the logic follows almost identically to the same part of the proof of \cref{thm:rbk}. The only difference is in how we show that $\overline{P} x \neq 0$ for any $x \in \range(A^\top)$. Like before, we start by noting that there must exist a row index $i$ such that $a_i^\top x \neq 0$. We then bound
\begin{equation}
x^\top \overline{P} x = \expect{S \sim \mathbf{U}(m,k)}{x^\top A_S^\top  (A_S A_S^\top + \lambda k I)^{-1} A_S x} \geq \frac{k}{m} \cdot \frac{(a_i^\top x)^2}{\norm{A^\top A + \lambda k I}_2} > 0,
\end{equation}
where the intermediate step uses the facts that $\norm{A_S A_S^\top + \lambda k I} = \norm{A_S^\top A_S + \lambda k I} \leq \norm{A^\top A + \lambda k I}$, that row $i$ is chosen with probability $k/m$, and that $x^\top A_S^\top A_S x \geq x^\top a_i a_i^\top x$ when $i \in S$.  It follows that $\overline{P} x \neq 0$ and so indeed $\range(\overline{P}) = \range(A^\top)$.

With the assumptions verified, \cref{lemma:expect} can be applied using $r=0$, $s=T$ to yield
\begin{equation}
\expect{}{x_T} - \xrh = (I - \overline{P})^T (x_0 - \xrh).
\end{equation}
This directly implies 
\begin{equation}
\norm{\expect{}{x_T} - \xrh} \leq (1 - \alpha)^T \norm{x_0 - \xrh}.
\end{equation}
using \cref{lemma:range,lemma:P} and the proven properties of $\overline{P}.$

The conditions of \cref{lemma:reblock} are also satisfied, the results of which provide the conditions for \cref{lemma:tail} with $B = \norm{x_0 - \xrh}^2$, $\tilde{V} = \revv{\frac{2}{\alpha} \expectE_{S \sim \rho} \norm{A_S^\top (A_S A_S^\top + \lambda k I)^{-1} \rrh_S}^2.}$ \cref{lemma:tail} then \revv{yields the desired bound for the tail averages.}

The next step is to bound the condition number $\kappa(\overline{W})$. \revv{Recall from the proof of \Cref{thm:rbk} that $\norm{A_S A_S^\top} \leq k \cdot \max_i \norm{a_i}^2$, so that} 
\begin{equation}
\lambda k I \preceq A_S A_S^\top + \lambda k I \preceq (\revv{\max_i \norm{a_i}^2} + \lambda) \cdot k I.
\end{equation}
Now, when $S$ is sampled uniformly, $\expect{S \sim \mathbf{U}(m,k)}{I_S^\top I_S} = \frac{k}{m} I$. It follows that
\begin{equation}
\overline{W}= \expect{S \sim \mathbf{U}(m,k)}{I_S^\top (A_S A_S^\top + \lambda k I)^{-1} I_S} \succeq \frac{1}{( \max_i \revv{\norm{a_i}^2} + \lambda)m}I, 
\end{equation}
\begin{equation}
\overline{W}= \expect{S \sim \mathbf{U}(m,k)}{I_S^\top (A_S A_S^\top + \lambda k I)^{-1} I_S} \preceq \frac{1}{\lambda m}I.
\end{equation}
Putting these together implies
\begin{equation} \label{eq:reblock_cond}
\kappa(\overline{W}) \leq 1 + \frac{1}{\lambda} \cdot \max_i \revv{\norm{a_i}^2}.
\end{equation}

\revv{
The bounds on the weighted residual and the bias follow from \Cref{lemma:res} with the bias bound taking the standard form 
\begin{equation}
\|x^{(\rho)}-x^*\|
 \le \sqrt{\kappa(\overline{W})-1} \cdot \|A^+\| \cdot \|r\|. \label{eq:reblock_bias}
\end{equation}
Since $\max_i \norm{a_i}^2 \leq \norm{A}^2$, the combination of \cref{eq:reblock_bias,eq:reblock_cond} implies the looser bound
\begin{equation}
\|x^{(\rho)}-x^*\| \leq \frac{1}{\sqrt{\lambda}} \cdot \kappa(A) \cdot \norm{r}.
\end{equation}}

\revv{
Lastly, consider the variance $V=\mathbb E\|A_S^\top (A_SA_S^\top+\lambda k I)^{-1} r_S^{(\rho)}\|^2$.
Since for $\sigma\ge0$ one has $\sigma /(\sigma+\lambda k)^2\le 1 / (4\lambda k)$, it holds
\[
\big\|A_S^\top (A_SA_S^\top+\lambda k I)^{-1}\big\|^2
=\lambda_{\max}\!\Big((A_SA_S^\top)(A_SA_S^\top+\lambda k I)^{-2}\Big)
\ \le\ \frac{1}{4\lambda k}.
\]
Thus
\[
V\ \le\ \frac{1}{4\lambda k} \cdot \mathbb E\|r_S^{(\rho)}\|^2
= \frac{1}{4\lambda m} \cdot\|r^{(\rho)}\|^2 \leq \frac{1}{4 \lambda} \cdot \kappa(\overline{W}) \cdot \frac{\norm{r}^2}{m}.
\]}
\end{proof}

\section{Noisy Linear Least-squares\label{app:noisy}}
In this section we consider the special case in which the inconsistency in the problem is generated by zero-mean noise. This case is more general than the case of Gaussian data, but less general than the case of arbitrary inconsistency. 
\revv{In particular, we consider the statistical linear regression problem
\begin{equation} \label{eq:noisy_prob}
x^* = \argmin{x \in \Rea^n} \expect{a \sim \mathbf{A}, z \sim \mathbf{Z}}{(a^\top x - (a^\top x^* - z))^2}
\end{equation}
where $\mathbf{A}$ is arbitrary, $\mathbf{Z}$ satisfies $\expect{z \sim \mathbf{Z}}{z}=0$,  and $x^* \in \Rea^n$. 
}
This scenario might arise, for example, when making noisy measurements of $x^*$ along the directions $a$. 

\revv{The RBK-U and ReBlocK-U algorithms are applied to \cref{eq:noisy_prob} by generating \(k\) independent samples per iteration from the data distributions $\mathbf{A},\mathbf{Z}$.}
The result is that RBK and ReBlocK both converge (in a Monte Carlo sense) to the ordinary least-squares solution rather than a weighted solution. This can be viewed as an advancement over previous results on noisy linear systems, such as \cite{needell2010randomized}, which show only convergence to a finite horizon even for arbitrarily large noisy systems. The advance comes from treating the noise model explicitly and applying tail averaging to the algorithm.

\begin{theorem}
\label{thm:rbk_noisy}
Consider the RBK-U algorithm applied to the noisy linear least-squares problem \eqref{eq:noisy_prob}. Then the results of \cref{thm:rbk} hold and \revv{$x^{(\rho)} = x^*$}.
\end{theorem}

\begin{theorem}
\label{thm:reblock_noisy}
Consider the ReBlocK-U algorithm applied to the noisy linear least-squares problem \eqref{eq:noisy_prob}. Then the results of \cref{thm:reblock} hold and \revv{$x^{(\rho)} = x^*$}.
\end{theorem}

\begin{proof}[Unified proof of \cref{thm:rbk_noisy,thm:reblock_noisy}]
\revv{Recall that the weighted solution $\xrh$ is characterized by
\begin{equation}
\overline{P} \xrh = A^\top \overline{W} b = \overline{P} x^* + A^\top \overline{W} r = \overline{P}x^* + \expect{}{A_S^\top M(S) r_S}
\end{equation}
where $r_S = A_S x^* - b_S$. Based on the definition of the noisy equation \cref{eq:noisy_prob}, each entry of the residual satisfies $r_i = a_i^\top x^* - (a_i^\top x^* - z_i) = z_i$ and is thus distributed exactly according to $\mathbf{Z}$. In particular each entry of $r_S$ is independent from $A_S$ and the other entries of $r_S$ and has a mean of zero. It follows that 
\begin{equation}
\expect{}{A_S^\top M(S) r_S} = \expect{}{A_S^\top M(S)} \expect{}{r_S} = 0,
\end{equation}
which completes the proof that $x^{(\rho)} = x^*$.}
\end{proof}

\section{Sampling from a Determinantal Point Process \label{app:dpp}}

In Section \ref{sec:reblock}, we have presented the convergence of ReBlocK to a weighted least-squares solution under the uniform sampling distribution $\rho = \mathbf{U}(m,k)$. However, the explicit convergence rate $1-\alpha$ is still unknown. Motivated by recent work on Kaczmarz algorithms with  determinantal point process (DPP) sampling \cite{derezinski2024solving}, we now show that under DPP sampling, ReBlocK converges to the ordinary least-squares solution and the convergence of ReBlocK can have a much better dependence on the singular values of $A$ than mSGD. 

The sampling distribution that we consider is  $\rho= k\text{-}{\rm DPP}(AA^\top + \lambda k I)$, namely
\begin{equation}
\label{eq:dpp}
\Pr[S] = \frac{\det (A_S A_S^\top + \lambda k I)}{\sum_{|S'| = k} \det (A_{S'}A_{S'}^\top + \lambda k I)}.
\end{equation}
Our analysis is similar to parts of the nearly concurrent work \cite{derezinski2025}, though \cite{derezinski2025} does not incorporate a fixed size $k$ for their regularized DPP distribution.

\begin{theorem} \label{thm:dpp}
Consider the ReBlocK algorithm with k-DPP sampling, namely $M(A_S) = (A_S A_S^\top + \lambda k I)^{-1}$ and $\rho= k\text{-}{\rm DPP}(AA^\top + \lambda k I)$. Let  $\alpha = \sigma_{\rm min}^+(\overline{P})$ and assume $x_0 \in \range(A^\top)$.  Then the expectation of the ReBlocK iterates $x_T$ converges to $x^*$ as
\begin{equation}
\norm{\expect{}{x_T} - x^*} \leq (1 - \alpha)^T \norm{x_0 - x^*}.
\end{equation}
Furthermore, the tail averages $\overline{x}_T$ converge to $x^*$ as 
\begin{equation}
\expectE \norm{\overline{x}_T - x^*}^2 \leq 2 (1 - \alpha)^{T_b+1} \norm{x_0 - x^*}^2 \\
 + \revv{\frac{4}{\alpha^2 (T-T_b)} V}
\end{equation}
\revv{with $V = \expectE_{S \sim \rho} \norm{A_S^\top (A_S A_S^\top + \lambda k I)^{-1} r_S}^2$}.
\revv{The variance term is bounded by 
\begin{equation}
V \leq \frac{1}{4 \lambda} \cdot \max_i r_i^2.
\end{equation}
}

Finally, for any $\ell < k$ the convergence parameter $\alpha$ satisfies 
\begin{equation} \label{eq:alpha}
\alpha \geq  \frac{k-\ell}{(k-\ell) +  \kappa_\ell^2(A) + (m+k-2\ell)\frac{\lambda \revv{k}}{\sigma_{n_r}^2}},
\end{equation}
where $\kappa_{\ell}^2(A) := \sum_{j > \ell}^{n_r} \sigma_j^2(A)/\sigma_{n_r}^2(A)$ and $\sigma_1(A) \geq \dots \geq \sigma_{n_r}(A) > 0= \sigma_{n_{r+1}} (A) = \dots = \sigma_n(A)$ are the singular values of $A$. 
\end{theorem}

By choosing $\ell =1$ and a small value of $\lambda$, the above theorem indicates that when $n_r=n$ the value of $\alpha$ is at least on the order of $k/\kappa^2_{\rm dem}(A)$ where $\kappa_{\rm dem}:= \|A\|_F / \sigma_{\min}(A)$ is the Demmel condition number. To understand the meaning of the bound more generally, suppose additionally that $\norm{a_i}=1$ for all $i$ so that $\norm{A}_F^2 = m$. Then, choosing $\ell = k/2$ we can rewrite \eqref{eq:alpha} as 
\begin{equation}
\begin{split} 
\label{eq:alpha-dpp-alterantive}
\alpha &\geq \frac{k}{k + 2\kappa_{k/2}^2(A) + 2m \lambda \revv{k}/\sigma_n^2} \\
&\approx \frac{k}{2\kappa^2_{k/2}(A) + 2 \lambda \revv{k} \kappa_{\rm dem}^2(A)}.
\end{split}
\end{equation}
As noted in \Cref{sec:gauss_data}, the convergence parameter for mSGD is bounded by
\[ \alpha^{\rm SGD} \leq k/\kappa_{\rm dem}^2(A). \]
Hence, when the singular values of $A$ decay rapidly and $\lambda \revv{k}$ is significantly smaller than one, ReBlocK with DPP sampling can have a much faster convergence rate than mSGD.

Our bound for the variance of the tail-averaged estimator is crude and could likely be improved using techniques such as those of \cite{derezinski2025}. Nonetheless, it is interesting to note that the dependence of the variance on the residual vector can be worse in the case of DPP sampling than in the case of uniform sampling. This is because the DPP distribution can potentially sample the residual vector in unfavorable ways. It is thus unclear whether DPP sampling is a good strategy for arbitrary inconsistent systems, even if it can be implemented efficiently

\begin{proof}
The majority of the proof follows similarly to the proofs of \cref{thm:rbk,thm:reblock}, and we only highlight the differences, of which there are three: the fact that $\xrh = x^*$, the value of the parameter $\alpha$, and the bound on the variance term in the tail-averaged bound.

We begin by verifying that $\xrh = x^*$. For a vector $u \in \mathbb{R}^{m}$, we denote its elementary symmetric polynomial by
$p_{k}(u):=\sum_{S \in\binom{[m]}{k}} \prod_{i \in S} u_i$, where $\binom{[m]}{k}$ is the set of all subsets of size $k$ from the set $[m]=\{1,\ldots,m\}$. Applying equation (5.3) of \cite{derezinski2024solving} by using $B = AA^\top + \lambda \revv{k} I$ leads to
\begin{equation} \label{eq:53}
\overline{W} = \expect{S \sim \rho}{I_S^\top (I_S (AA^\top + \lambda \revv{k} I) I_S^\top)^{-1} I_S} = \frac{U {\rm diag}(p_{k-1}(q_{-1}), \ldots, p_{k-1}(q_{-m})) U^\top}{p_{k}(q)},
\end{equation}
where $A = U\Sigma V^\top$ denotes the \revh{full} singular value decomposition of $A$, and $q_i = \sigma_i^2 + \lambda \revv{k}$ when $i \leq n_r$ and $q_i = \lambda \revv{k}$ otherwise, with $\sigma_1 \geq \dots \geq \sigma_{n_r}$ representing the sorted singular values and $n_r$ the rank of $A$. Denote by $\overline{W}= UDU^\top$ with the diagonal matrix $D:= \frac{{\rm diag}(p_{k-1}(q_{-1}), \ldots, p_{k-1}(q_{-m}))}{p_{k}(q)} \succ 0$. Then, by the definition $x^{(\rho)} = {\rm argmin}_{x \in \mathbb{R}^n} \|Ax - b\|_{\overline{W}}^2$, we have
\begin{equation} \label{eq:normal-1} A^\top \overline{W} Ax^{(\rho)} = A^\top \overline{W}b. \end{equation}
Note that $A^\top \overline{W} A = V \revh{\Sigma^\top} D \Sigma V^\top$ and $A^\top\overline{W} = V\revh{\Sigma^\top} DU^\top b$. \revh{Also, because $\Sigma \in \mathbb{R}^{m\times n}$ and $D \in \mathbb{R}^{m \times m}$ are diagonal, we have $\Sigma^\top D = \hat{D} \Sigma^\top$, where $\hat{D} :=D[1:n,1:n]$ is the leading principal submatrix of $D$. Substituting this into \eqref{eq:normal-1} gives 
$V\hat{D} \Sigma^\top \Sigma V^\top x^{(\rho)} = V\hat{D}\Sigma^\top U^\top b$.}
Thus, we can apply $V^\top$ on the left and use $D \succ 0$ to conclude
\[ \revh{\Sigma^\top \Sigma} V^\top x^{(\rho)} = \revh{\Sigma^\top} U^\top b. \]
Applying $V$ now on the left implies $A^\top A x^{(\rho)} = A^\top b$, i.e., $x^{(\rho)} \in {\rm argmin}_{x \in \mathbb{R}^n} \|Ax - b\|^2$. Since $\xrh \in \range(A^T)$ this implies that $\xrh = x^*$.

Regarding $\alpha$, plugging \eqref{eq:53} into the definition of $\overline{P}$ leads to
\begin{equation} \label{eq:P}
\overline{P} = A^\top \overline{W} A = V \Sigma^\top \frac{{\rm diag}(p_{k-1}(q_{-1}), \ldots, p_{k-1}(q_{-m}))}{p_{k}(q)} \Sigma V^\top.
\end{equation}

We now follow closely the logic of the proof of Lemma 4.1 of \cite{derezinski2024solving}, making appropriate modifications to handle the fact that in \revv{ReBlocK} we invert $AA^\top + \lambda \revv{k}  I$ rather than just $AA^\top$. Note also that through \eqref{eq:P} we have explicitly verified that $\range(\overline{P}) = \range(A^T)$, which is the needed condition for the contraction properties to hold in our convergence bound. 

First, for any $\ell < k$ we can construct an approximation matrix 
\begin{equation}
B_\ell = U {\rm diag}\paren{q_1, \ldots q_m} U^\top +\frac{1}{k-\ell}  \paren{\sum_{j>\ell}^m q_j} I.
\end{equation}
Note that we have used a mildly simpler and looser approximation than \cite{derezinski2024solving} by replacing $\frac{k - \ell - 1}{k - \ell}$ with $1$ in the first term. 

The same arguments as in \cite{derezinski2024solving} then imply that
\begin{equation}
\alpha = \sigma_{\min}^+(\overline{P}) \geq \sigma_{\min}^+ (A^\top B_\ell^{-1} A).
\end{equation}
Furthermore we have
\begin{align*}
\sigma_{\min}^+ (A^\top B_\ell^{-1} A) &= \sigma_{\min}^+  \paren{V \Sigma^\top U^\top B_\ell^{-1} U \Sigma V^\top} \\
&= \sigma_{\min} \paren{V {\rm diag} \paren{\frac{\sigma_1^2}{q_1 + \frac{1}{k - \ell} \sum_{j>\ell}^m q_j}, \ldots, \frac{\sigma_{n_r}^2}{q_{n_r} + \frac{1}{k - \ell} \sum_{j>\ell}^m q_j}} V^\top}\\
&= \frac{\sigma_{n_r}^2}{q_{n_r} + \frac{1}{k - \ell} \sum_{j > \ell}^m q_j} \\
&=  \frac{\sigma_{n_r}^2}{\sigma_{n_r}^2  + \frac{1}{k - \ell} \sum_{j > \ell}^{n_r} \sigma_j^2 +  \frac{m+k-2\ell}{k - \ell} \lambda \revv{k}} \\
&= \frac{k-\ell}{(k-\ell) + \kappa_\ell^2(A) + (m+k-2\ell)\frac{\lambda \revv{k}}{\sigma_{n_r}^2}}.
\end{align*}
This completes the bound for $\alpha$.

Finally, we consider the bound for the variance term in the tail-averaged bound. Following the proof of \cref{thm:reblock} we can obtain an initial bound of
\begin{equation}
\revv{V \leq \frac{1}{4 \lambda k} \cdot \expectE_{S \sim \rho} \norm{r_S}^2.}
\end{equation}
Unfortunately, in the case of DPP sampling, we have no guarantee on how $S$ samples the residual vector $\rrh$. Thus, the best we can do is uniformly bound $\norm{r_S}^2 \leq k \cdot \max_i r_i^2$. This yields the bound in the theorem.
\end{proof}

In this supplement we provide more details on the numerical examples throughout the paper. The code to run the experiments can be found at \url{https://github.com/ggoldsh/block-kaczmarz-without-preprocessing}. 

\section{Experiments with random matrices\label{app:synth}}
In this section we provide the details of the synthetic experiments from Sections 3.1, 4.1, and 5.1. 
All of these experiments have $m=10^5$, $n=10^2$, and $k=30$. The experiments differ only in how the matrix $A$ is generated. Thus, we first describe the matrix construction for each case separately, then describe the rest of the set-up in a unified manner. 

Beginning with the experiments of \Cref{sec:rbk_num}, for the left panel of \cref{fig:gaussian_results} we generate the matrix $A$ by setting each entry independently as $A_{ij} \sim \mathcal{N}(0,1)$. For the right panel, we construct $A = GU$ where $G \in \Rea^{m \times n}$ has independent standard normal entries and and $U$ has $\sigma_i = 1/i^2$ and random orthonormal singular vectors. The condition number of the matrix $A$ is $\kappa(A) \approx 1$ in the left panel and $\kappa(A) \approx 10^4$ in the right panel.

Regarding \Cref{sec:rbk_fail_num,sec:reblock_num}, the examples these sections are constructed as discretizations of underlying continuous problems, which is both a realistic scenario and serves to ensure that $A$ contains many nearly singular blocks. In both cases, we first generate an $n \times n$ matrix $C$. For the left panel we set $C=I$, whereas for the right panel we set $C$ with singular values $\sigma_i = 1/\revv{i}$ and random orthonormal singular vectors. Once $C$ is constructed, we define a set of $n$ functions $f_1, \ldots, f_n$ by 
\begin{equation}
f_j(s) = \sum_{\ell = 1}^n C_{j \ell} T_\ell(s),
\end{equation}
where $T_\ell$ is the $\ell^{\rm th}$ Chebyshev polynomial of the first kind. The functions $f$ correspond to the columns of $A$.

Next, we construct a vector $v$ of $m$ coordinates uniformly spaced across the interval $[-1, 1]$, namely
\begin{equation}
v_i = \paren{-1 + 2 \frac{i-1}{m-1}}.
\end{equation}
Finally, the matrix $A$ is designed as
\begin{equation}
A_{ij} = f_j(v_i),
\end{equation}
so that column $j$ of $A$ is a discretized representation of the function $f_j$. The condition number of the resulting matrices $A$ are $\kappa(A) \approx 11$ in the left panel and $\kappa(A) \approx \revv{450}$ in the right panel.

For all experiments in  \Cref{sec:rbk_fail_num,sec:rbk_num,sec:reblock_num} the vectors $b$ are constructed by setting
\begin{equation}
\label{eq:gen_b}
b_i = a_i^\top y + z_i
\end{equation}
where $y \sim \mathcal{N}(0,I_n)$ and $z_i \sim \mathcal{N}(0, \sigma)$. The noise level is set to $\sigma=1e-2$ except for in \Cref{fig:inconsistency} where the noise level is explicitly varied. The initial guess is $x_0 = 0$ and the number of iterations is \revv{$T=10^5$}, which is equivalent to \revv{thirty} passes over the data. The step size for minibatch SGD is constant within each run and has been tuned independently for each example, specifically to be as large as possible without introducing any signs of instability, up to a factor of $2$. The value of $\lambda$ for ReBlocK has been set to $\lambda = 0.001$ and is not optimized on a per-example basis. \revv{ The reported quantity is the relative solution error $\norm{x - x^*} / \norm{x^*}$}. These experiments with random matrices are carried out in double precision on a 1.1 GHz Quad-Core Intel Core i5 CPU.

\section{Natural Gradient Experiments\label{app:nn}}
In this section we describe in detail the numerical experiments of \Cref{sec:ngd}. The training problem for the neural network is a simple function regression task on the unit interval. The target function is chosen to be periodic to avoid any consideration of boundary effects, and the neural network is correspondingly designed to be periodic by construction. Specifically, the target function is constructed as
\begin{equation} 
f(s) = q(\sin(2 \pi s)),
\end{equation}
with the polynomial $q$ defined as
\begin{equation}
q(s) = \frac{1}{\sqrt{d}} \sum_{\ell = 1}^{d} c_\ell T_\ell(s)
\end{equation}
for $d = 30$ and each $c_\ell$ chosen randomly from a standard normal distribution. The resulting function is pictured in \cref{fig:target}.

\begin{figure}[htbp]
\centering
\includegraphics[width=0.4\textwidth]{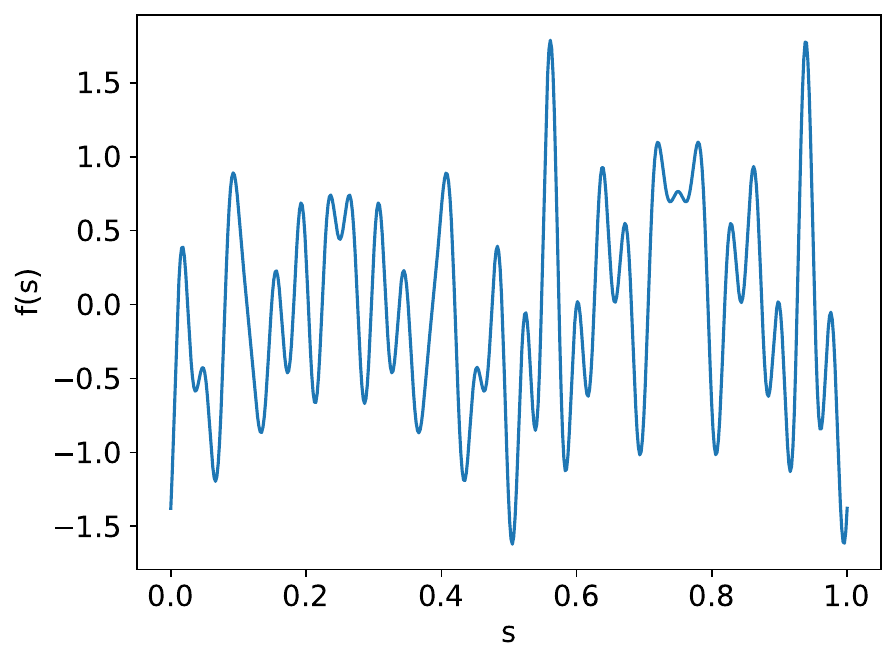} 
\caption{Target function for neural network training.}
\label{fig:target}
\end{figure}

The neural network model is a simple periodic ResNet \cite{he2016deep} with 5 layers. For input $s \in \Rea$, the network outputs $f(s) \in \Rea$ are given by
\begin{equation}
y_1 = \tanh(W_1 \sin (2 \pi s) + b_1),
\end{equation}
\begin{equation}
y_i = y_{i-1} + \tanh(W_i y_{i-1} + b_i) ;\; i = 2,3,4,
\end{equation}
\begin{equation}
f(s) = W_5 y_4 + b_5.
\end{equation}
The intermediate layers have dimensions $y_i \in \Rea^{500}$ and the weight matrices $W_i$ and bias vectors $b_i$ have the appropriate dimensions to match. The weights are initialized using a Lecun normal initialization and the biases are all initialized to zero. The parameters are collected into a single vector $\theta \in \Rea^{753,001}$ for convenience, leading to the neural network function $f_\theta(s)$. 

The loss function is defined as in \eqref{eq:func_learn} and the network is trained using subsampled natural gradient descent, as described for example in \cite{ren2019efficient}, with a batch size of $N_b = 500$, a Tikhonov regularization of $\lambda = 0.01$, and a step size of $\eta = 0.5$. This corresponds to the parameter update
\begin{equation}
\theta_{t+1} = \theta_t - \eta J_S^\top (J_S J_S^\top + N_b \lambda I)^{-1} [f_\theta - f]_S,
\end{equation}
where $S$ represents the set of $N_b$ sample points. The setting of $N_b=500$ is intended to represent something close to a full batch training regime, which is only practical since our function regression problem lives in a very compact, one-dimensional domain. The resulting training curve is presented in \cref{fig:nn_train}.

\begin{figure}[htbp]
\centering
\includegraphics[width=0.4\textwidth]{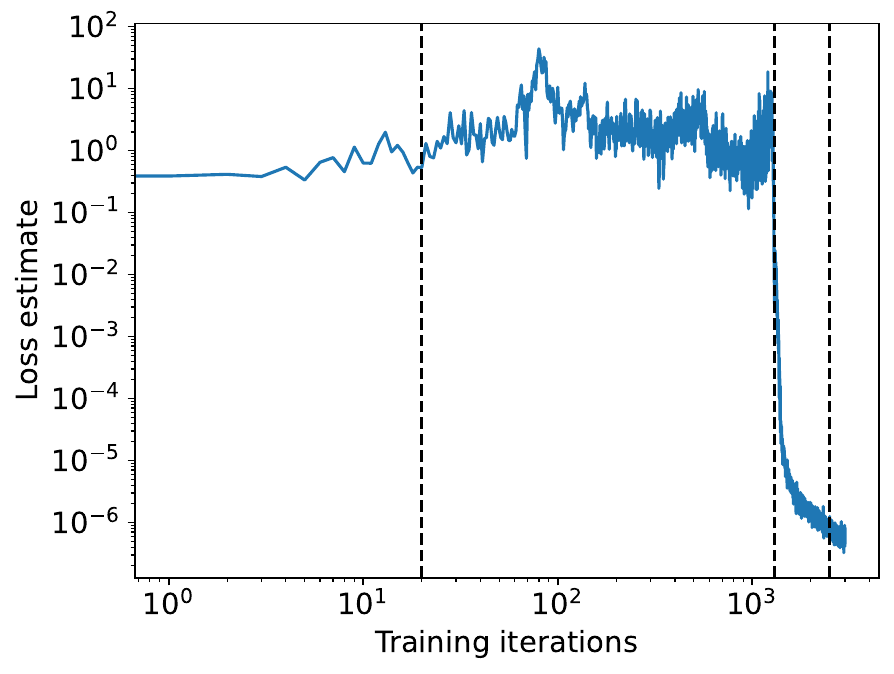} 
\caption{Training curve for the neural network function regression example. The three vertical lines indicate the training snapshots that are used to generate the least-squares problems studied in \cref{fig:nn}.}
\label{fig:nn_train}
\end{figure}

From this training run, three snapshots are taken and used to generate the least-squares problems for \cref{fig:nn,fig:per_it,fig:lambdas,fig:ks,fig:it_speed}, following \cref{eq:ngd_ls}. The first snapshot is from the ``pre-descent'' phase before the loss begins to decrease, the second is from the ``descent'' phase during which the loss decreases rapidly, and the third is from the ``post-descent'' phase when the decay rate of the loss has slowed substantially. The snapshots are indicated by the vertical dotted lines in \cref{fig:nn_train} and form the basis of the experiments of \cref{fig:nn,fig:per_it,fig:lambdas,fig:ks,fig:it_speed}. The batch size  of $k=50$ used in \cref{fig:nn} is meant to represent the realistic scenario when each iteration uses too few samples to thoroughly represent the target function. Furthermore, the continuous problems are treated directly by uniformly sampling $k$ points from the domain $[0,1]$ at each iteration and calculating the network outputs and gradients at these points.

The initial guess is $x_0=0$ in every case and the step size for minibatch SGD is constant within each run and has been tuned independently for each example, specifically to be as large as possible without introducing any signs of instability, up to a factor of $2$. The value of $\lambda$ for ReBlocK has been set to $\lambda = 0.001$ unless otherwise specified, and is not optimized on a per-example basis. The reported quantity is the relative residual $\tilde{r} = \norm{J x - [f_\theta - f]} / \norm{f_\theta - f}$, which measures how well the function-space update direction $J x$ agrees with the function-space loss gradient $f_\theta - f$. This quantity is estimated at each iteration using the available sample points.

\section*{Disclaimer}
This report was prepared as an account of work sponsored by an agency of the
United States Government. Neither the United States Government nor any agency thereof, nor
any of their employees, makes any warranty, express or implied, or assumes any legal liability
or responsibility for the accuracy, completeness, or usefulness of any information, apparatus,
product, or process disclosed, or represents that its use would not infringe privately owned
rights. Reference herein to any specific commercial product, process, or service by trade name,
trademark, manufacturer, or otherwise does not necessarily constitute or imply its
endorsement, recommendation, or favoring by the United States Government or any agency
thereof. The views and opinions of authors expressed herein do not necessarily state or reflect
those of the United States Government or any agency thereof.

\bibliographystyle{siamplain}
\bibliography{references}
\end{document}